\documentclass[accepted]{uai2021} %

\usepackage[american]{babel}

\usepackage[dvipsnames]{xcolor}
\usepackage{natbib} %
\usepackage{mathtools} %
\usepackage{booktabs} %
\usepackage{tikz} %

\title{Tractable Computation of Expected Kernels}

\author[1]{Wenzhe Li\thanks{Authors contributed equally. This research was performed while W.L.~was visiting UCLA remotely.}}
\author[2]{Zhe Zeng$^{*}$}
\author[2]{Antonio Vergari}
\author[2]{Guy Van den Broeck}
\affil[1]{%
    Tsinghua University    
}
\affil[2]{%
    University of California, Los Angeles
    \begin{tabular}{c}
    \texttt{scott.wenzhe.li@gmail.com}, \texttt{\{zhezeng, aver, guyvdb\}@cs.ucla.edu}
    \end{tabular}
}

\usepackage{algorithm}
\usepackage{algorithmicx}
\usepackage[noend]{algpseudocode}
\usepackage{amsmath,amsfonts,amsthm}
\usepackage{graphicx, subcaption}
\usepackage{xspace}
\usepackage{stmaryrd}
\usepackage{hyperref}
\usepackage{amssymb}
\usepackage{enumitem}
\usepackage{bm}
\usepackage{wasysym}
\usepackage[capitalise]{cleveref}
\usepackage{subfiles}

\newcommand{\X}{\ensuremath{\mathbf{X}}\xspace}
\newcommand{\xdomain}{\ensuremath{\mathcal{X}}\xspace}
\newcommand{\ydomain}{\ensuremath{\mathcal{Y}}\xspace}
\newcommand{\Xs}{\ensuremath{\mathbf{X_s}}\xspace}
\newcommand{\Xc}{\ensuremath{\mathbf{X_c}}\xspace}
\newcommand{\x}{\ensuremath{\mathbf{x}}\xspace}
\newcommand{\xs}{\ensuremath{\mathbf{x_s}}\xspace}
\newcommand{\xc}{\ensuremath{\mathbf{x_c}}\xspace}
\newcommand{\s}{\ensuremath{\mathbf{s}}\xspace}
\newcommand{\co}{\ensuremath{\mathbf{c}}\xspace}
\newcommand{\expec}{\ensuremath{\mathbb{E}}\xspace}

\newcommand{\bigO}{\mathcal{O}}
\newcommand{\steinop}{\ensuremath{\mathcal{T}_p}\xspace}
\newcommand{\trace}{\ensuremath{\mathit{tr}}\xspace}
\newcommand{\ckdsd}{CKDSD\xspace}
\newcommand{\doublesum}{\ensuremath{M}\xspace}
\newcommand{\real}{\ensuremath{\mathbb{R}}\xspace}

\newcommand{\mmd}{\ensuremath{\mathit{MMD}}\xspace}

\newcommand{\wl}[1]{}
\newcommand{\zz}[1]{}
\newcommand{\guy}[1]{}
\newcommand{\av}[1]{}

\definecolor{lacamlilac} {RGB} {107,93,153}
\definecolor{gold2} {RGB} {255, 130, 0}
\definecolor{petroil2} {RGB} {36, 165, 175}
\definecolor{pink4} {HTML} {EC407A}
\definecolor{olive4} {HTML} {7CB342}

\newcommand{\vv}[1]{\ensuremath{\boldsymbol{#1}}\xspace}

\newcommand{\xp}{\ensuremath{\mathbf{x}^\prime}\xspace}

\newcommand{\R}{\ensuremath{\mathbb{R}}\xspace}
\newcommand{\F}{\ensuremath{\mathcal{F}}\xspace}
\newcommand{\ksd}{\ensuremath{\mathbb{S}}\xspace}

\newcommand{\p}{{p}\xspace}
\newcommand{\q}{{q}\xspace}
\newcommand{\f}{{f}\xspace}
\newcommand{\g}{{g}\xspace}
\newcommand{\ch}{\ensuremath{\mathsf{in}}\xspace}
\newcommand{\leftn}{\ensuremath{\mathsf{L}}\xspace}
\newcommand{\rightn}{\ensuremath{\mathsf{R}}\xspace}

\newcommand{\score}{\vv{s}_p\xspace}
\newcommand{\scorei}{\vv{s}_{p, i}\xspace}
\newcommand{\sampleid}{\ensuremath{\mathcal{S}}\xspace}
\newcommand{\condid}{\ensuremath{\mathcal{C}}\xspace}
\newcommand{\idx}[1]{{(#1)}}
\DeclareMathOperator*{\argmin}{arg\,min}

\newtheorem{thm}{Theorem}[section]
\newtheorem{cor}[thm]{Corollary}
\newtheorem{mydef}[thm]{Definition}

\newtheorem{pro}[thm]{Proposition}

\begin{document}
\maketitle

\begin{abstract}
Computing the expectation of kernel functions is a ubiquitous task in machine learning, with applications from classical support vector machines to exploiting kernel embeddings of distributions in probabilistic modeling, statistical inference, causal discovery, and deep learning.
In all these scenarios, we tend to resort to Monte Carlo estimates as expectations of kernels are intractable in general.
In this work, we characterize the conditions under which we can compute expected kernels exactly and efficiently,
by leveraging recent advances in probabilistic circuit representations.
We first construct a circuit representation for kernels and propose an approach to such tractable computation.
We then demonstrate possible advancements for kernel embedding frameworks by exploiting tractable expected kernels to derive new algorithms for two challenging scenarios:
1) reasoning under missing data with kernel support vector regressors;
2) devising a collapsed black-box importance sampling scheme.
Finally, we empirically evaluate both algorithms and show that they outperform standard baselines on a variety of datasets.

\end{abstract}

\section{Introduction}
\label{sec: intro}
Kernel functions have been prominent in the machine learning community for decades. 
Kernels provided a convenient notion of inner product for high-dimensional feature maps~\citep{cortes1995support,scholkopf1998nonlinear} and have been extended to represent distributions as elements in a reproducing kernel Hilbert space~(RKHS). 
They have contributed to various fundamental tasks including sample testing~\citep{gretton2012kernel,DBLP:conf/nips/JitkrittumX0FG17}, group anomaly detection~\citep{muandet2013one} and 
causal discovery~\citep{chen2014causal}.

One fundamental computation that naturally arises in these kernel-embedding based frameworks is to compute the expectations of a kernel function w.r.t. distributions over its inputs. 
For instance, it arises in integral probability metrics~(IPMs)~\citep{muller1997integral} when the functional space is chosen as an RKHS and distributions are characterized by 
their
kernel embeddings.
However, such expectations are computationally hard in general and most existing methods resort to Monte Carlo estimators for approximation.

In this paper, we investigate how to derive a tractable algorithm to compute these kernel expectations, thus enabling the aforementioned frameworks to perform exact inference without relying on unreliable approximations.
We do so by leveraging recent advances in tractable probabilistic modeling.
Specifically, our algorithmic contribution will take advantage of representing both the kernels and the input distributions participating in the expectation as \textit{circuits}. 

Circuit representations~\citep{tutorial-pc,choi2020pc} reconcile and abstract from the different graphical and syntactic representations of both classical tractable probabilistic models such as mixture models (e.g., mixtures of Gaussian distributions), bounded-treewidth graphical models~\citep{koller2009probabilistic,meila2000learning} and more recent ones such as 
probabilistic circuits~\citep{choi2020pc,vergari2021compositional} like
arithmetic circuits~\citep{darwiche2003differential}, probabilistic sentential decision diagrams (PSDDs)~\citep{kisa2014probabilistic}, sum-product networks (SPNs)~\citep{poon2011sum}, and cutset networks~\citep{rahman2014cutset}.
As such, our analysis within the framework of circuit representations will help trace the boundaries of tractable computations of kernel expectations, delivering a general and efficient scheme that can be flexibly applied to many kernel-embedding scenarios and different tractable probabilistic model formalisms.

For this representation language, we characterize under which structural constraints on kernel functions and probability distributions the expectations of kernels can be computed exactly and efficiently.
We show how kernel functions can be represented as circuits with the requisite structural properties, and construct a recursive algorithm that delivers the tractable computation of their expectation in time polynomial in the size of the circuit representations.

Moreover, we demonstrate how the tractable computation of expected kernels can serve as a powerful tool to derive novel kernel-based algorithms on two challenging tasks when using 
 kernel embeddings to represent features as well as distributions. 
The first is to enable kernel support vector regressors to deal with missing data by computing their  expected predictions~\citep{anderson2011expected,khosravi2019tractable}. 
In the second, we derive a novel collapsed black-box importance sampling scheme using the kernelized Stein discrepancy~\citep{liu2016black} for efficient approximate inference over factor graph models that do not have a tractable representation. 
We compare each algorithm with existing baselines on different real-world datasets and problems, showing that our exact expected kernels yield  better inference performance.

\section{Expected Kernels}
\label{sec: expected kernels}
We use uppercase letters $X$ for random variables and lowercase letters $x$ for their assignments. Analogously, we denote a set of random variables in bold uppercase $\X$ and their assignments in bold lowercase $\x$. 
The domain of variables $\X$ is denoted by $\xdomain$.
The cardinality of $\xdomain$ is denoted by $|\xdomain|$.

We are interested in the modular operation of computing expected kernels. This task naturally arises in various kernel-embedding based frameworks.
\begin{mydef}[Expected Kernel]
\label{def: expected kernel}
Given two distributions $\p$ and $\q$ over variables $\X$ on domain $\xdomain$, and a 
positive definite 
kernel function $k: \xdomain \times \xdomain \rightarrow \real$, 
the expected kernel, that is,
the expectation of the kernel function $k$ with respect to the distributions $p$ and $q$ is defined as follows.
\begin{equation}
\doublesum_{k}(p, q) := \expec_{\x \sim p, \xp \sim q}[k(\x, \xp)]
\label{eq:exp-kernels}
\end{equation}
\end{mydef}
Expected kernels are omnipresent in 
machine learning.
For instance, one of the most well-known 
IPMs, the squared maximum mean discrepancy (MMD)~\citep{gretton2012kernel}
is defined as
$\mmd^2[\mathcal{H}, p, q] = \doublesum_k(p, p) + \doublesum_k(q, q) - 2 \doublesum_k(p, q)$ and measures the distance between two distributions $p$ and $q$ whose embeddings via a kernel $k$ live in a RKHS $\mathcal{H}$.
However,
the computation cost of expected kernels is prohibitive in general, even for distributions that are tractable for other inference scenarios, as the next theorem illustrates.
\begin{thm}
There exist representations of distributions $p$ and $q$ that are tractable for computing marginal, conditional, and maximum a-posteriori (MAP) probabilities, yet computing the expected kernel of a simple kernel $k$ that is the Kronecker delta is already \#P-hard.
\label{thm: hardness for expected kernels}
\end{thm}
Concretely, we show that this is true for probabilistic circuit representations, which unify several tractable probabilistic model representations.
We defer the proof of the above statement to Section~\ref{sec: tractable computation of expected kernels} after circuits are introduced.

The most commonly adopted solution to estimating Equation~\ref{eq:exp-kernels} and circumventing its computational challenge is to approximate it by sampling.
Instead, we are interested in defining a large model class guaranteeing its tractable computation and thus providing an efficient algorithm to compute it exactly.
We will show that this is possible by leveraging circuit representations of functions. %
In summary, 
we first adopt the probabilistic circuit representations for distributions, and further build a circuit representation for kernel functions to allow an exact computation of the expected kernels to be described in circuit operations.
Then, we exploit the structural constraints on circuits such that the computational complexity can be bounded to be polytime in the size of circuits.
The necessary background on circuits is presented in Section~\ref{sec: circuit representation} and the tractable computation of expected kernels is demonstrated in Section~\ref{sec: tractable computation of expected kernels}.

\paragraph{Expected Kernels in Action}
Our proposed tractable computation of expected kernel can be applied to expressive distribution families and it can potentially lead to new advances in kernel-based frameworks. 
To demonstrate this, we show how tractable expected kernels give rise to novel algorithms for two challenging tasks,
where the kernels serve as \emph{embeddings for features} in one algorithm, and as \emph{embeddings for distributions} in the other,
covering the two most popular usages of kernel functions.
The first one is to reason about kernel-based support regression models 
in the presence of missing features.
The second one is to perform black-box importance sampling with collapsed samples,
where expected kernels are leveraged to obtain the kernelized discrepancy between collapsed samples, which further gives the optimal importance weights. We will show the detailed descriptions of the proposed algorithms in Section~\ref{sec: expected kernels in action} and their empirical evaluation in Section~\ref{sec: empirical evaluation}.

\section{Circuit Representation}
\label{sec: circuit representation}

\begin{figure*}[t]
\centering
\begin{subfigure}[t]{0.32\textwidth}
\centering
\includegraphics[page=2,width=0.95\textwidth]{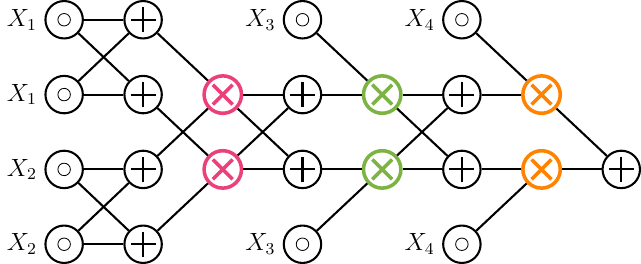}
\caption{a structure-decomposable PC}
\label{fig: pc example 1}
\end{subfigure}
\begin{subfigure}[t]{0.32\textwidth}
\centering
\includegraphics[page=3,width=0.95\textwidth]{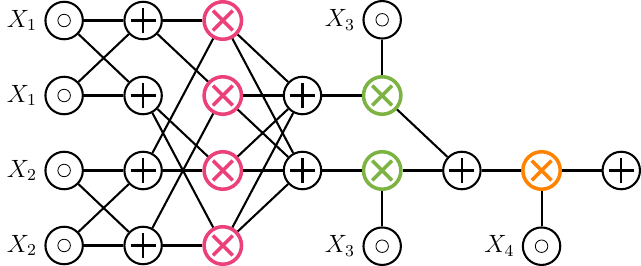}
\caption{a PC compatible with the PC in (a)
}
\label{fig: pc example 2}
\end{subfigure}
\begin{subfigure}[t]{0.35\textwidth}
\centering
\includegraphics[page=4,width=0.95\textwidth]{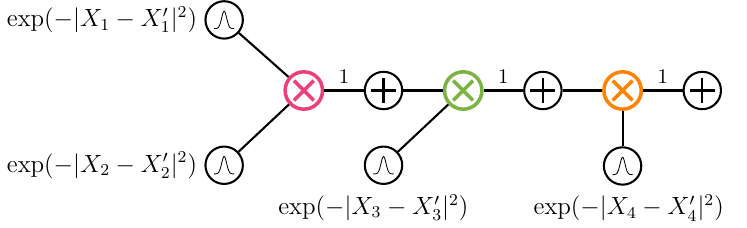}
\caption{a KC kernel-compatible with (a) and (b)}
\label{fig: kc example}
\end{subfigure}
\caption{\emph{Examples of circuit representations.} Units in the computational graph include sum units, product units, univariate input distribution units represented with a circle and labeled by their scopes, and non-linear input function units represented with a curve and labeled by the input functions. Sum parameters are omitted for visual clarity. The feed-forward evaluation (input before outputs) is intended from left to right. The rightmost unit is the output of the circuit. All product nodes are colored according to their scopes: $\{X_1, X_2\}$ in \textcolor{pink4}{pink}, $\{X_1, X_2, X_3\}$ in \textcolor{olive4}{green}, and $\X$ in \textcolor{gold2}{orange}.}
\end{figure*}

Circuits are parameterized representations of functions as computational graphs.
They provide a language to characterize the tractability of function operations in terms of structural constraints over these computational graphs.
Next we first introduce circuits and their properties.

\begin{mydef}[Circuit]
\label{def: circuits}
A circuit $\f$ over variables $\X$ is a parameterized computational graph encoding a function $\f(\X)$
and comprising three kinds of computational units:
\textit{input}, \textit{product}, and \textit{sum}. 
Each inner unit $n$ (i.e., product or sum unit) receives inputs from some other units, denoted $\ch(n)$.
Each unit $n$ encodes a function $\f_{n}$ as follows:
\begin{equation*}
{\f}_n(\phi(n))=
\begin{cases}
    l_n({\phi(n)}) &\text{if $n$ is an input unit} \\
    \prod_{c\in\ch(n)} \f_c(\phi(c)) &\text{if $n$ is a product unit} \\
    \sum_{c\in\ch(n)} \theta_{c} \f_c(\phi(c)) &\text{if $n$ is a sum unit}
\end{cases}
\label{eq: EVI}
\end{equation*}
where $\theta_{c}\in\R$ are the parameters associated with each sum node,
and input units encode parameterized functions $l_n$ over variables $\phi(n) \subseteq \X$, also called their \textit{scope}.
The scope of an inner unit
is the union of the scopes of its inputs:  $\phi(n)=\bigcup_{c \in \ch(n)}\phi(c)$.
The final output unit (the root of the circuit) encodes $\f(\X)$.
\end{mydef}

Circuits can be understood as compact representations of polynomials, whose indeterminates are the functions encoded by the input units.
They are assumed to be simple enough to allow locally tractable computations which further forms global operations with tractability guarantees.

Most well-known circuit classes are various forms of probabilistic circuits~(PCs)~\citep{tutorial-pc,choi2020pc}.
PCs provide a unified framework where probabilistic inference operations are cleanly mapped to the circuit representations.
As such, they abstract from the many graphical formalism for tractable probabilistic models, from classical shallow mixtures~\citep{koller2009probabilistic,meila2000learning} to more recent deep variants~\citep{poon2011sum,peharz2020einsum}.
Specifically,
a PC
encodes a (possibly unnormalized) probability distribution over a collection of variables in a recursive manner.

\begin{mydef}[Probabilistic Circuits]
A PC on domain $\xdomain$ is a circuit encoding a non-negative function $\p: \xdomain \rightarrow \real^{\geq 0}$.
\end{mydef}

A circuit $\p$ can be evaluated
in time linear in its size denoted by $|\p|$, i.e., the number of edges in its computational graph.
For example,
computing $\p(\X=\x)$ in a PC can be done in a \textit{feedforward} way, evaluating input units before outputs, and hence in time linear in the size of the PC.

W.l.o.g., we will assume that units in circuits alternate layerwise between sum and product units and that every product unit receives only two inputs.
Both requirements can be easily enforced in any circuit structure with a polynomial increase in its size~\citep{peharz2020einsum,vergari2015simplifying}.
Furthermore, in this work we focus on discrete variables.
For conciseness, we denote the circuit by the same notation as the function that it represents, for instance, a PC $\p$ refers to the circuit representation of the distribution $\p$.

\paragraph{Properties of Circuits}
The tractability of computing quantities of interest involving the function encoded in a circuit, also called queries, can be characterized by \emph{structural constraints} on the computational graph of its circuit \citep{darwiche2002knowledge}. 
Next we introduce the structural properties that will be sufficient for the tractable computation of the expected kernels.
We refer the interested reader to~\citet{choi2020pc} for additional properties enabling other tractable inference scenarios.

\begin{mydef}[Smoothness]
A circuit is \textit{smooth}, if for every sum node $n$, 
its inputs $\ch(n)$ share the same scope, 
i.e., $\forall c, c^\prime\in\ch(n), \phi(c)=\phi(c^\prime)$.
\end{mydef}
Some examples of smooth circuits are mixture models: they comprise a single sum node over tractable input distributions that have to share the same scope.
For example, a Gaussian mixture model (GMM) can be represented as a smooth circuit with a single sum unit and several input units, each of which encodes a (multivariate) Gaussian density defined over the same set of variables.

\begin{mydef}[Determinism]
A circuit is \textit{deterministic} if the inputs of every sum unit have disjoint supports.
\end{mydef}

Determinism in PCs enables the tractable computation of MAP inference.
In this work, determinism will play a role in exactly computing the KSD between discrete distributions (see Corollary \ref{cor: tractable kdsd}).
\begin{mydef}[Decomposability]
A circuit is \textit{decomposable}, if for every product node $n$, its inputs $\ch(n)$ have disjoint scopes,
i.e., $\forall c, c^\prime\in\ch(n), c \neq c^\prime: \phi(c)\cap\phi(c^\prime)=\emptyset$.
\end{mydef}
Decomposable product nodes encode local factorizations.
For example, a decomposable product node $n$ over variables $\X$ with inputs from two units can be written as $\f_{n}(\X)=\f_{\leftn}(\X_{\leftn})\f_{\rightn}(\X_{\rightn})$, where
$\X_{\leftn}$ and $\X_{\rightn}$
form a partition of $\X$.
Taken together, smoothness and decomposability are 
sufficient and necessary
for performing tractable integration over arbitrary sets of variables in a single feedforward pass, 
which allows to compute marginals and conditionals 
in time linear in the circuit size~\citep{choi2020pc}.
To characterize tractable kernel expectations, we will need the multiple circuits participating in it to have product units that decompose their scopes in a ``synchronized'' way.
This property, called compatibility, is formalized recursively as follows.

\begin{mydef}[Compatibility]
Two circuits $\f$ and $\g$ are compatible if 
(i) they are smooth and decomposable, and
(ii) for any pair of product units $n \in \f$ and $m \in \g$ that share the same scope, they decompose in the same way, i.e.,
for every unit $c \in \ch(n)$, there must exist a unique unit $c^\prime \in ch(m)$ such that $ \phi(c)=\phi(c^\prime)$.
\end{mydef}

\begin{mydef}[Structured-decomposability]
A circuit is structured-decomposable if it is compatible with itself.
\end{mydef}
Notice that structured-decomposable circuits are a strict subclass of decomposable circuits.
An example of a structured-decomposable PC is shown in Figure~\ref{fig: pc example 1}.
The way that a structured-decomposable circuit hierarchically partitions its scope can be compactly represented by a graph called
\emph{vtree}~\citep{pipatsrisawat2008new}, \emph{pseudo-forest}~\citep{jaeger2004probabilistic} or \emph{pseudo-tree}~\citep{dechter2007and}.
In a nutshell, compatible structured-decomposable circuits conform to the same hierarchical partitioning over their variables.
Figure~\ref{fig: pc example 1} and Figure~\ref{fig: pc example 2} show two compatible PCs.
This additional requirement enables also the tractable computation of moments of predictive models~\citep{khosravi2019tractable} and the probability of logical constraints~\citep{bekker2015tractable,choi2015tractable}.

\paragraph{Construction of PCs}
As mentioned before,  several classes of tractable probabilistic graphical models (PGMs) including Chow-Liu trees~\citep{chow1968approximating} and hidden Markov models (HMMs)~\citep{rabiner1986introduction} can be represented as compact PCs with certain structural properties.
The process of translating one graphical representation into a circuit  is called \textit{compilation} and has received much attention in the literature~\citep{chavira2005compiling,darwiche2011sdd}.
In particular, \citet{shen2016tractable} propose a very efficient compilation scheme that compiles a factor graph into a structured-decomposable PC by first representing each factor as a PC and then multiplying them together.

Besides compiling PCs from other tractable models, we can also directly learn PCs from data~\citep{lowd2012learning,rooshenas2014learning,peharz2020einsum}.
Recently learning algorithms tailored towards structured-decomposable PCs have been proposed~\citep{LiangXAI17,DangPGM20}. 
For our experiments we will employ \textsc{Strudel}~\citep{DangPGM20} for its simplicity and speed.

\section{TRACTABLE COMPUTATION OF EXPECTED KERNELS}
\label{sec: tractable computation of expected kernels}

Computing expected kernels
is a \#P-hard problem in general. It involves summation over exponentially many states in the distribution space.
We first provide a formal proof for the hardness statement provided in Theorem~\ref{thm: hardness for expected kernels}.

\begin{proof}{[\textbf{Theorem~\ref{thm: hardness for expected kernels}}]}
Consider the case when $p$ and $q$ are both structured-decomposable and deterministic probabilistic circuits, and the positive definite kernel $k$ is a Kronecker delta function defined as $k(\x, \xp) = 1$ if and only if $\x = \xp$. 
Then computing the expected kernel $\doublesum_k(p, q)$ is equivalent to computing the quantity $\sum_{\x \in \xdomain} p(\x) q(\x)$,
which has been shown to be \#P-hard by \cite{vergari2021compositional}.
Therefore, computing the expected kernel is \#P-hard.
\end{proof}

From the proof we can tell that
mild structural constraints on circuits are not enough to reduce the computational complexity.
We provide another proof in Appendix where a pair of probabilistic circuits with different constraints is considered.
Together they show that it is highly challenging to derive sufficient structural constraints to guarantee tractability.

The aim of this section is to investigate under what structural constraints on circuits an exact and efficient computation of expected kernels is possible.
But before we characterize tractability in the circuit language, we need to consider \textit{whether also kernels can be represented as circuits}. 
To answer this question
we define kernel circuits (KCs) to be the circuit representations of kernel functions that measure similarities between input pairs defined on the kernel domain.
\begin{mydef}
A KC on domain $\xdomain \times \xdomain$ is a circuit encoding a symmetric kernel function $k: \xdomain \times \xdomain \rightarrow \real^+$.
\end{mydef}
\paragraph{Remark.}
To verify that a given KC is positive definite, 
it is sufficient to verify that the input units are positive definite kernels and that the sum parameters are positive
since the positive definite kernel family is closed under summation and product.
Moreover, it can be done tractably in time linear in the number of input units in the KC.

Figure~\ref{fig: kc example} shows an example kernel circuit.
We further define the left (resp. right) projection of a KC given $\x \in \xdomain$ to be $k(\cdot, \x): \xdomain \rightarrow \real^+$ (resp. $k(\x, \cdot): \xdomain \rightarrow \real^+$).
Intuitively, for the tractability of expected kernels,
the KC should have its structure conform to the distributions that it measures,
which allows the measurement to be broken down into basic ones along the circuit.
Next, we characterize the structural constraints on KCs suitable for such a computation.

\begin{mydef}[Kernel Compatibility]
\label{def: kernel-compatibility}
Let $p$ and $q$ be a pair of compatible circuits.
A kernel circuit $k(\X, \X^\prime)$
is kernel-compatible with the circuit pair $p(\X)$ and $q(\X^\prime)$ if 
\begin{enumerate}[label=\roman*),noitemsep,topsep=0pt]
\item the kernel circuit $k$ is smooth and decomposable, and
\item the left and right projections of $k$ are compatible with circuit $p$ and $q$ respectively for any $\x \in \xdomain$.
\end{enumerate}
\end{mydef}

For example, the KC shown in Figure~\ref{fig: kc example} is kernel-compatible with the circuit pair shown in Figure~\ref{fig: pc example 1} and Figure~\ref{fig: pc example 2}.
Intuitively, a KC with 
kernel compatibility
measures the similarity between the two probability distributions
in a hierarchical way.

\textit{Note that many commonly used kernels have a circuit representations that exhibits kernel compatibility}.
These include several exponentiated forms such as the radial basis function kernel (RBF) and the exponentiated Hamming kernel.
To see how, consider an RBF kernel $k(\X, \X^\prime) = \exp(- \sum_{i = 1}^4 |X_i - X^\prime_i|^2)$. 
It can be represented by a KC with one product unit connected to four input units each of which represents the basic function $\exp(- |X_i - X^\prime_i|^2)$.
Given a pair of compatible PCs $p$ and $q$ as in Figure~\ref{fig: pc example 1} and Figure~\ref{fig: pc example 2}, we can always transform the KC of an RBF kernel into a circuit compatible with $p$ and $q$ by ``splitting'' its product unit into intermediate products that are compatible with the product units in $p$ and $q$ and by introducing dummy sum units receiving single inputs and with parameter $\theta=1$.
The resulting KC is shown in Figure~\ref{fig: kc example}.

Next we show our main result: kernel compatibility is sufficient to guarantee the tractability of expected kernels.
\begin{thm}
\label{thm: double sum complexity}
Let $\p$ and $\q$ be a pair of compatible PCs, and $k$ be a kernel circuit.
If $k$ is kernel-compatible with $\p$ and $\q$,
the expected kernel $\doublesum_{k}(\p, \q)$ can be computed exactly in $\bigO(|\p||\q||k|)$ time.\footnote{As the algorithm will show, this is not a tight bound and in practice the effective number of recursive calls will be much smaller than $|\p||\q||k|$.}
\end{thm}

The proof is by construction. Intuitively, the computation of expected kernels can be recursively ``broken down'' along the circuit structures, until we reach collections of input units for which we can assume the integrals in the expectations to be tractably computed.
The next proposition shows this recursion over circuits whose outputs are sums.

\begin{algorithm}[!tp]
\caption{$\doublesum_{k_l}(\p_n, \q_m)$ --- Computing the expected kernel
} 
\label{alg: double-sum} 
\textbf{Require:} Two compatible PCs $\p_n$ and $\q_m$, and a KC $k_l$ that is kernel-compatible with the PC pair $\p_n$ and $\q_m$.
\begin{algorithmic}[1]
\If{$n, m, l$ are input units} \\
    \quad \Return $\doublesum_{k_l}(\p_n, \q_m)$
\ElsIf{$n, m, l$ are sum units}\Comment{cf. Prop. \ref{pro: recursive-sum-nodes}} \\
    \quad \Return
    $\sum_{i\in\ch(n), j\in\ch(m), c\in \ch(l)}
    \theta_{i}\delta_{j}\gamma_c
    \,\doublesum_{k_c}(\p_i, \q_j)$
\ElsIf{$n, m, l$ are product units}\Comment{cf. Prop. \ref{pro: recursive-product-nodes}} \\
    \quad \Return 
    $\doublesum_{k_\leftn}(\p_{n_\leftn}, \q_{m_\leftn}) \cdot \doublesum_{k_\rightn}(\p_{n_\rightn}, \q_{m_\rightn})$
\EndIf
\end{algorithmic} 
\end{algorithm}

\begin{pro}
\label{pro: recursive-sum-nodes}
Let $\p_{n}$ and $\q_{m}$ be two smooth probabilistic circuits over variables $\X$ whose output units $n$ and $m$ are sum units,
denoted by
$\p_{n}(\X) = \sum_{i\in\ch(n)}\theta_{i}\p_{i}(\X)$ and 
$\q_{m}(\X) = \sum_{j\in\ch(m)}\delta_{j}\q_{j}(\X)$ respectively. 
Let $k_{l}$ be a kernel circuit with its output unit being a sum unit $l$, denoted by 
$k_l(\X) = \sum_{c \in \ch(l)} \gamma_{c} k_{c}(\X)$.
Then it holds that
\begin{equation}
    \doublesum_{k_l}(\p_n, \q_m) = 
    \sum_{i\in\ch(n)}\theta_{i}\sum_{j\in\ch(m)}\delta_{j} 
    \sum_{c \in \ch(l)} \gamma_{c}
    \,\doublesum_{k_c}(\p_i, \q_j).
\end{equation}
\end{pro}

This way, the expected kernel can be computed by the weighted sum of a number of simpler expected kernel computations over the input units.
Analogously, the expected kernel computation can be broken down at the product units as follows thanks to compatibility.

\begin{pro}
\label{pro: recursive-product-nodes}
Let $\p_n$ and $\q_m$ be two compatible probabilistic circuits over variables $\X$ whose output units $n$ and $m$ are product units, denoted by $\p_{n}(\X) = \p_{n_\leftn}(\X_{\leftn})\p_{n_\rightn}(\X_{\rightn})$ and $\q_{m}(\X) = \q_{m_\leftn}(\X_{\leftn})\q_{m_\rightn}(\X_{\rightn})$.
Let $k_l$ be a kernel circuit that is kernel-compatible with the circuit pair $\p_n$ and $\q_m$ with its output unit being a product unit denoted by
$k_l(\X, \X^\prime) = k_{\leftn}(\X_{\leftn}, \X_{\leftn}^\prime) k_{\rightn}(\X_{\rightn}, \X_{\rightn}^\prime)$.
Then it holds that
\begin{equation*}
    \doublesum_{k_l}(\p_n, \q_m) = \doublesum_{k_\leftn}(\p_{n_\leftn}, \q_{m_\leftn}) \cdot \doublesum_{k_\rightn}(\p_{n_\rightn}, \q_{m_\rightn}).
\end{equation*}
\end{pro}
Lastly, for the base cases of the recursion we can have that either both $\p$ and $\q$ comprise a single input distribution (sharing the same scope), or one of them is an input distribution and the other a sum unit.\footnote{The other unit cannot be a product unit otherwise 
compatibility
would be violated.}
The first case is easily computable in polytime by the assumption in~\cref{thm: double sum complexity}.
Note that this assumption is generally easy to meet as the double summation in $\doublesum_{k}(\p_n, \q_m)$ for input distributions can be computed in polytime by enumeration, since input distributions have limited scopes (generally univariate) and ${\p}(\x){\q}(\xp)k(\x,\xp)$ can be computed in closed form for decomposable kernels $k$ and commonly used distributions such as discrete distributions as in our case. 
The second corner case reduces to the first when noting that computing $\doublesum_{k}(\p_n, \q_m)$ for an input distribution and a mixture of input distributions reduces to computing a weighted sum of expectations followed by applying \cref{pro: recursive-sum-nodes}. 
Algorithm~\ref{alg: double-sum} summarizes the whole computation of the expected kernel $\doublesum_{k}$, which requires only 
polynomial complexity when caching repeated calls.

As direct results of Theorem~\ref{thm: double sum complexity}, we show that two 
common kernelized discrepancies in reproducing kernel Hilbert space~(RKHS) can be tractably computed if the same structural constraints apply to the distributions and kernels.
\begin{cor}
\label{cor: tractable mmd}
Following the assumptions in Theorem~\ref{thm: double sum complexity},
the squared maximum mean discrepancy $\mmd[\mathcal{H}, p, q]$ in RKHS $\mathcal{H}$ associated with kernel $k$ as defined in \citet{gretton2012kernel} can be tractably computed
in time $\bigO(|\p||\q||k|)$.
\end{cor}

\begin{cor}
\label{cor: tractable kdsd}
Following the assumptions in Theorem~\ref{thm: double sum complexity}, if the probabilistic circuit $p$ further satisfies determinism,
the kernelized discrete Stein discrepancy~(KDSD)
$\mathbb{D}^2(q~\parallel~p) = \expec_{\x, \xp \sim q}[k_p(\x, \xp)]$
in the RKHS associated with kernel $k$ as defined in \citet{yang2018goodness} can be tractably computed.
\end{cor}

The computation of expected kernels by circuit operations allows us to compute the kernel-embedding based statistics exactly and efficiently. This further gives rise to interesting applications part of which will be shown in the next section.
We leave the further explorations on what other statistics will benefit from the proposed computation of expected kernels and what more applications will be inspired as future work.

\section{Expected Kernels in Action}
\label{sec: expected kernels in action}
In this section we will show how the tractable computation of expected kernels can be leveraged in 
1) \emph{kernel embedding for features} to derive an inference algorithm for support vector regression (SVR) under missing data;
2) \emph{kernel embedding for distributions} to
derive a collapsed estimator in black-box importance sampling~(IS).
We further demonstrate the effectiveness of both proposed expected-kernel based algorithms empirically in Section~\ref{sec: empirical evaluation}.

\subsection{SVR for Missing Data}

Support vector machines (SVMs) for classification and regression are widely used in machine learning~\citep{noble2006support}. SVMs' foundations have great theoretical appeal, and they are still widely used in practice.
How to deal with missing features in SVMs has been an active area of research~\citep{aydilek2013hybrid,saar2007handling,marlin2008missing}.

In this section, we aim to tackle missing features in 
SVR at deployment time from a principled probabilistic perspective, like in \citet{anderson2011expected}, but for a larger model class represented as circuits. We propose to leverage PCs to learn the joint feature distribution, and then exploit tractable expected kernels to efficiently compute the expected predictions of SVR models.
More formally, given a set of input variables $\X$ (features) with domain $\xdomain$ and a variable $Y$ (target) with domain $\ydomain$, and a kernel function $k$,
a kernelized SVR learns from a dataset $\{(\x^{(i)}, y^{(i)})\}_{i=1}^n$
to predict for new inputs with a function $f$ taking the form
\begin{equation}
\label{eq: SVR predictor}
    f(\X) = \sum_{i=1}^n w_i k(\x^{(i)}, \X) + b.
\end{equation}

Existing works to handle missing features at deployment time include imputation strategies that substitutes missing values with reasonable alternatives such as the mean or median, estimated from training data.
The imputation methods are typically heuristic and model-agnostic, and sometimes make strong distributional assumptions such as total independence of the feature variables. 
As demonstrated in~\citet{khosravi2019expect}, computing expected predictions is not only theoretically principled but practically effective. 

\begin{mydef}[Expected prediction]
Given a predictive model $f: \xdomain \rightarrow \ydomain$, a distribution $p(\X)$ over features $\X$ and a partial assignment $\xs$ for variables $\Xs \subset \X$, the expected prediction of $f$ w.r.t.\ $p$ is
\begin{equation}
    \expec_{\xc \sim p(\Xc | \xs)} [f(\x)],
\end{equation}
where $\Xc = \X \backslash \Xs$ and where $\x$ is the completed feature vector consisting of both $\xc$ and $\xs$.
\end{mydef}

Intuitively, the expected prediction of a SVR given a partial feature vector can be thought of as
reweighting all possible completions by their probability.
Expected prediction enjoys the  theoretical guarantee that it is consistent under both missing completely at random (MCAR) and missing at random (MAR) 
 mechanisms, if $f$ has been trained on complete data and is Bayes optimal~\citep{josse2019consistency}. 

\begin{pro}
\label{pro: svr tractability}
Given a SVR model $f$ with a KC $k$, and a structured-decomposable PC $\p$ for the feature distribution, the expected prediction of $f$ can be tractably computed
in time $\bigO(|k||\p|)$.
\end{pro}
\begin{proof}
The expected prediction of $f$ w.r.t. $p$ can be rewritten as a linear combination of expected kernels.
\begin{equation*}
    \expec_{\xc \sim p(\Xc | \xs)} [f(\x)] = \sum_{i=1}^n w_i \expec_{\xc \sim p(\Xc | \xs)} [k(\x,\x^{(i)})] + b.
\end{equation*}
Note that the task of computing the doubly expected kernel in Definition~\ref{def: expected kernel} subsumes the task of computing a singly expected kernel where one of the inputs to the kernel function is a constant vector $\x_i$ instead of a variable
and both Theorem~\ref{thm: double sum complexity} and Algorithm~\ref{alg: double-sum} apply here.
\end{proof}

\subsection{Collapsed Black-box Importance Sampling}
Black-box importance sampling (BBIS)~\citep{liu2016black} is a recently introduced algorithm to flexibly perform approximate probabilistic inference on intractable distributions.
By weighting samples from an arbitrary proposal as to minimize a kernelized Stein discrepancy (KSD), BBIS can accurately estimate continuous target distributions. 

In this section, we first show that the BBIS algorithm can be extended to discrete distributions by adopting a recently proposed kernelized discrete Stein discrepancy (KDSD)~\citep{yang2018goodness} that serves as the discrete counterpart for KSD.
We further show that the BBIS algorithm can be improved by using collapsed samples, which is made possible by the tractable computation of expected kernels.

We start with a brief overview of how to construct the KDSD.
For a finite domain $\xdomain$, a \emph{cyclic permutation} denoted by $\neg$ is a bijection associated with some ordering of elements in $\xdomain$ that maps an element in $\xdomain$ to the next one according to the ordering. A \emph{partial difference operator} $\Delta^*$ for any function $f$ on domain $\xdomain$ is defined as 
$\Delta^* f(\x) := (\Delta^*_1 f(\x), \cdots, \Delta^*_D f(\x))$, with
$\Delta^*_i f(\x) := f(\x) - f(\neg_i \x)$ for $i = 1, 2, \cdots, D$ with $D = |\X|$.
Now we are ready to define the (difference) score function, an important tool for determining a probability distribution.
The score function is defined as 
$\score(\x) := \Delta^* p(\x)/p(\x)$,
a vector-valued function with its $i$-th dimension being
$\scorei(\x) := \Delta^*_i p(\x)/p(\x)$.
Then the KDSD between two distributions $p$ and $q$ is defined as
\begin{equation}
    \mathbb{D}(q \parallel p) := \sup_{\vv{f} \in \mathcal{F}} \expec_{\x\sim q(\X)}[\steinop \vv{f}(\x)],
    \label{eq: sd}
\end{equation}
with the functional space $\F$ being RKHS associated with a strictly positive definite kernel $k$, and the operator $\steinop$ being the \emph{Stein difference operator} defined as
$\steinop \vv{f} := \score(\x) \vv{f}^\top - \Delta \vv{f}(\x)$.
The KDSD is a proper divergence measure in the sense that for any strictly positive distribution $p$ and $q$, the KDSD $\mathbb{D}(q \parallel p) = 0$ if and only if $p = q$ \citep{yang2018goodness}.
Moreover, a nice property of the KDSD is that even though it involves a variational optimization problem in its definition, it admits a closed-form representation as
\begin{equation}
    \ksd(q \parallel p) 
    := \mathbb{D}^2(q \parallel p)
    = \expec_{\x, \xp \sim q}[k_p(\x, \xp)],
    \label{eq: ksd}
\end{equation}
with the kernel function $k_p$ defined as
\begin{equation*}
\label{eq: kernel-p}
\begin{split}
    k_p(\x, \xp) 
    &= \score(\x)^\top k(\x, \xp) \score(\xp) - \score(\x)^\top \Delta^{\xp} k(\x, \xp) \\
    &\quad - \Delta^{\x} k(\x, \xp)^\top \score(\xp) + \trace(\Delta^{\x,\xp} k(\x, \xp)),
\end{split}
\end{equation*}
where the superscript $\x$ and $\xp$ of the difference operator specifies the variables that it operates on.

We can now proceed to propose a BBIS algorithm for categorical distributions.
Given a set of samples $\{\x^{(i)}\}_{i = 1}^n$ generated from some unknown proposal $q$ possibly from some black-box mechanism, 
Categorical BBIS computes the importance weights for the samples by minimizing 
the KDSD between $q$ and target distribution $p$ formulated as
\begin{equation}
\label{eq: continuous-optimal-weights}
    \vv{w}^* = \argmin_{\vv{w}} \left\{
    \vv{w}^\top \vv{K_p} \vv{w} \,\middle\vert\,
    \sum_{i = 1}^n w_i = 1, ~w_i \geq 0 \right\},
\end{equation}
where $\vv{K}_p$ is a Gram matrix with entries $[\vv{K}_p]_{ij} = k_p(\x^{(i)}, \x^{(j)})$ 
and $\vv{w} = (w_1, \cdots, w_n)$ is the weight vector. 
We prove that the BBIS for categorical distributions enjoys the same convergence guarantees as its continuous counterpart. Due to space constraints, we defer both the algorithm details and convergence proofs to the Appendix.

However, a computational bottleneck in BBIS limits its scalability, the construction of the Gram matrix.
We therefore propose a collapsed variant of BBIS to accelerate it by delivering equally good approximations with fewer samples.
Collapsed samplers, also known as \textit{cutset} or \textit{Rao-Blackwellised} samplers~\citep{casella1996rao}, improve over classical particle-based methods by limiting sampling to a subset of the variables while pairing it with some closed-form representation of a conditional distribution over the rest.

Specifically,
let $(\Xs, \Xc)$ be a partition for variables $\X$.
A \textit{weighted collapsed sample} for variables $\X$ takes the form of a triplet $(\xs, p(\Xc \mid \xs), w)$
where $\x_\s$ is an assignment for the sampled variables $\Xs$, $p(\Xc \mid \xs)$ is a conditional distribution over the \textit{collapsed set} $\Xc$, and $w$ the importance weight.
We now show how to distill a \textit{conditional} KDSD, in order to extend BBIS to the collapsed sample scenario. 
\begin{mydef}[Conditional KDSD]
\label{def: c-KDSD}
Assume given a strictly positive 
distribution $p$ and a strictly positive proposal distribution of the sampled set $q_\s$, where the variable subset $\X_\s$ defines the samples. The full distribution defined by the collapsed samples is $q(\x) = q_\s(\x_\s) p(\x_\co \mid \x_\s)$.
The conditional KDSD~(\ckdsd) 
is defined as the KDSD between distributions $p$ and $q$,
i.e., $\ksd_{\s}(q_\s \parallel p) := \ksd(q \parallel p)$.
\end{mydef}

\begin{pro}
\label{pro: kdsd}
The \ckdsd 
between the two positive distributions $p$ and $q$ admits a closed form as
\begin{equation}
    \begin{split}
        \ksd_{\s}(q_\s \parallel p)
        = \expec_{\x_\s, \xp_\s \sim q_\s(\X_{\s})}[k_{p, \s}(\x_\s, \xp_\s)],
    \end{split}
\end{equation}
where $k_{p, \s}$ denotes a conditional kernel function defined as
\begin{equation}
\label{eq: expected kernel kps}
    \begin{split}
        k_{p, \s}(\x_\s, \xp_\s) = 
        \mathop{\expec}_{
            \substack{\x_{\co} \sim p(\X_{\co} \mid \x_\s), \xp_{\co} \sim p(\X_{\co} \mid \xp_\s)}
        }
        \left[k_p(\x, \xp)\right].
    \end{split}
\end{equation}
\end{pro}

Similar to the optimization in Equation~\ref{eq: continuous-optimal-weights} for BBIS, 
given a set of collapsed samples $\{(\xs^{(i)}, p(\Xc \mid \xs^{(i)}))\}_{i = 1}^n$, the problem of computing importance weights can be cast as minimizing the empirical CKDSD between the collapsed samples and the target distribution $p$
as follows.
\begin{equation}
    \begin{split}
        \ksd_{\s}(\{\xs^{(i)}, w_i \} \parallel p)
        = \vv{w}^\top \vv{K}_{p, \s} \vv{w}
    \end{split}
\end{equation}
where $\vv{w}$ is the vector of sample weights and $\vv{K}_{p, \s}$ is the Gram matrix with entries $[\vv{K}_{p, \s}]_{ij} = k_{p, \s}(\xs^{(i)}, \xs^{(j)})$.
Now the key question is whether the conditional kernel function $k_{p, \s}$ can be computed tractably. We show that this is possible with the tractable computation of expected kernels.
\begin{pro}
\label{pro: tractable conditonal kernel function}
Let $p(\Xc \mid \xs)$ be a PC that encodes a conditional distribution over variables $\Xc$ conditioned on $\Xs = \xs$, and $k$ be a KC.
If the PC $p(\Xc \mid \xs)$ and $p(\Xc \mid \xs^\prime)$ are compatible and $k$ is kernel-compatible with the PC pair for any $\xs$, $\xs^\prime$, 
then the conditional kernel function $k_{p, \s}$ can be tractably computed.
\end{pro}
This finishes the construction of a BBIS scheme using the collapsed samples, 
which we name CBBIS. The complete algorithmic recipe for CBBIS is presented in Algorithm~\ref{alg: cbbis-pc}

\begin{algorithm}[!tp]
\caption{\textsc{CBBIS}($p, q_\s, k, n$)
} 
\label{alg: cbbis-pc} 
\textbf{Input:} target distribution $p$ over variables $\X$, black-box mechanism $q_\s$, kernel function $k$, number of samples $n$\\
\textbf{Output:} a set of weighted collapsed samples
\begin{algorithmic}[1] 
\State Sample $\{\xs^\idx{i}\}^{n}_{i=1}$ from $q_\s$
\For{ $i = 1, \ldots, n$}
\State Compile $p(\Xc \mid \xs^\idx{i})$ into a PC
\Comment cf.~Sec. \ref{sec: inference performance}
\EndFor
\For{ $i = 1,\ldots,n$}
\For{ $j = 1,\ldots,n$}
\State $[\vv{K}_{p}]_{ij} = k_{p, \s}(\xs^\idx{i}, \xs^\idx{j})$ 
\Comment cf.~Prop. \ref{pro: tractable conditonal kernel function} 
\EndFor
\EndFor
\State $\vv{w}^* = \argmin_{\vv{w}}\left\{ \vv{w}^\top \vv{K}_p \vv{w} \,\middle\vert\, \sum_{i = 1}^n w_i = 1, ~w_i \ge 0 \right\}$
\State\Return $\{(\xs^\idx{n}, p(\Xc \mid \xs^\idx{n}), w^*_i)\}_{i=1}^{n}$
\end{algorithmic} 
\end{algorithm}

\begin{figure*}[!t]
    \centering
        {\includegraphics[width=0.245 \textwidth]{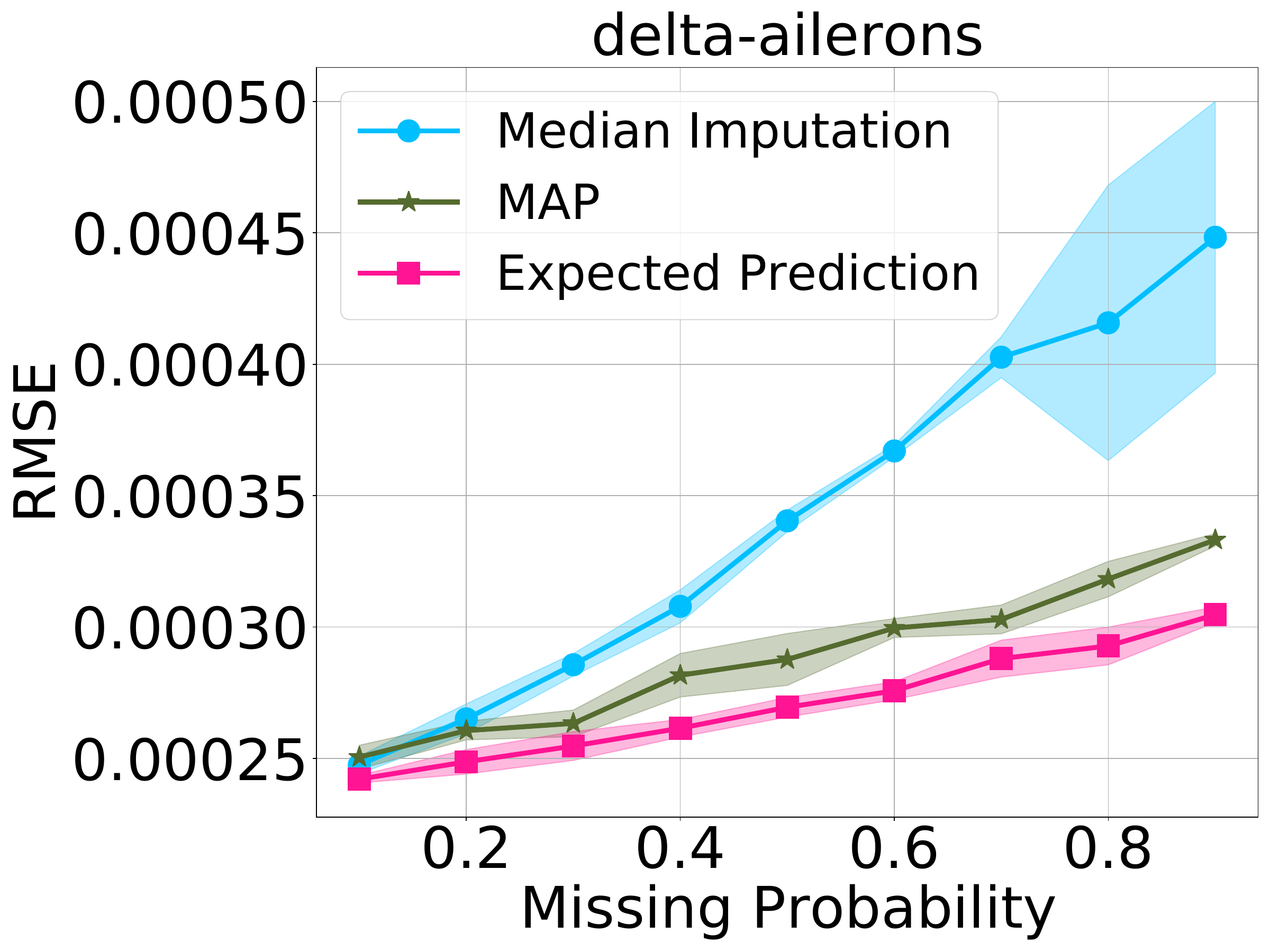}}
        {\includegraphics[width=0.245 \textwidth]{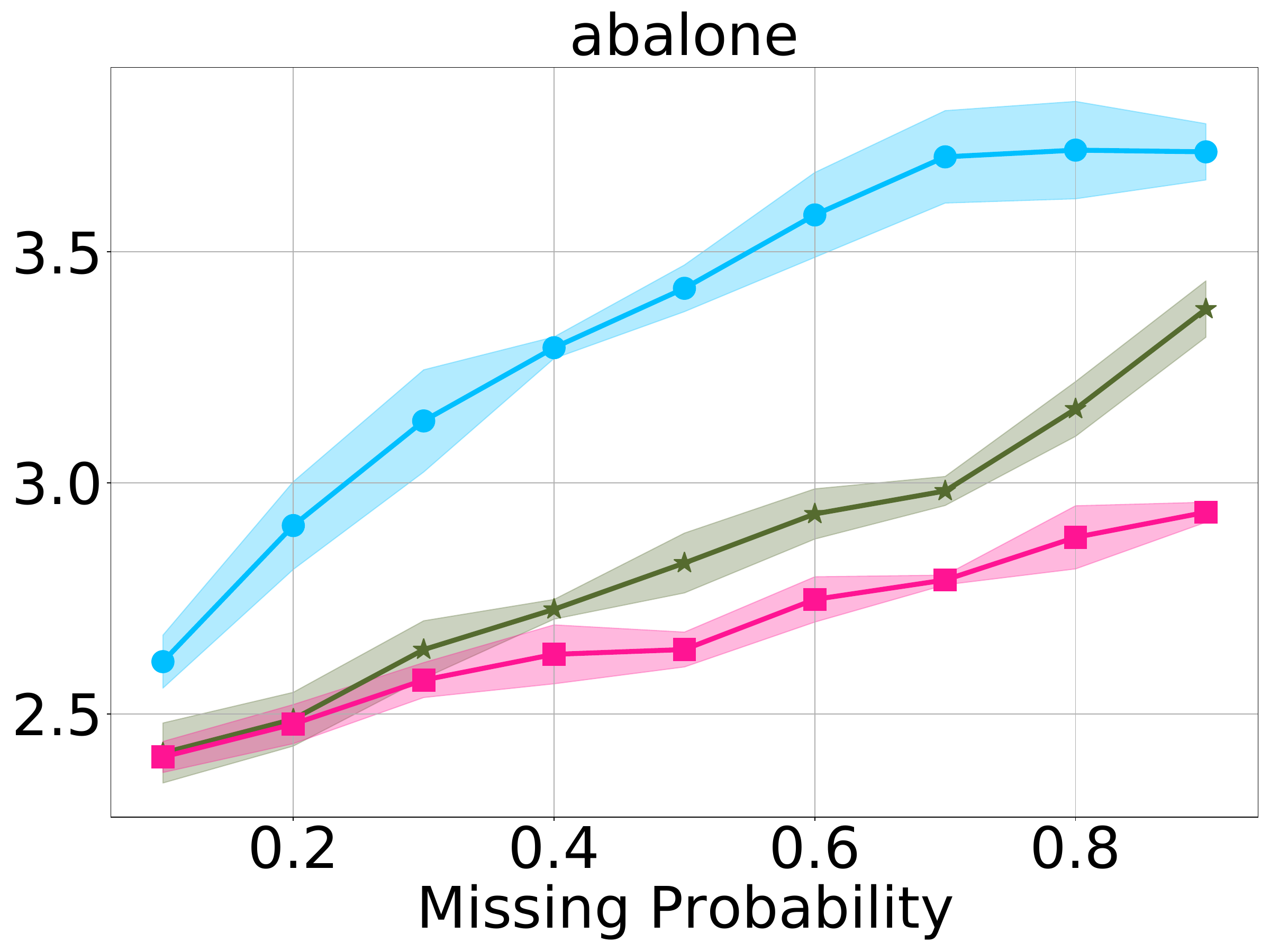}}
        {\includegraphics[width=0.245 \textwidth]{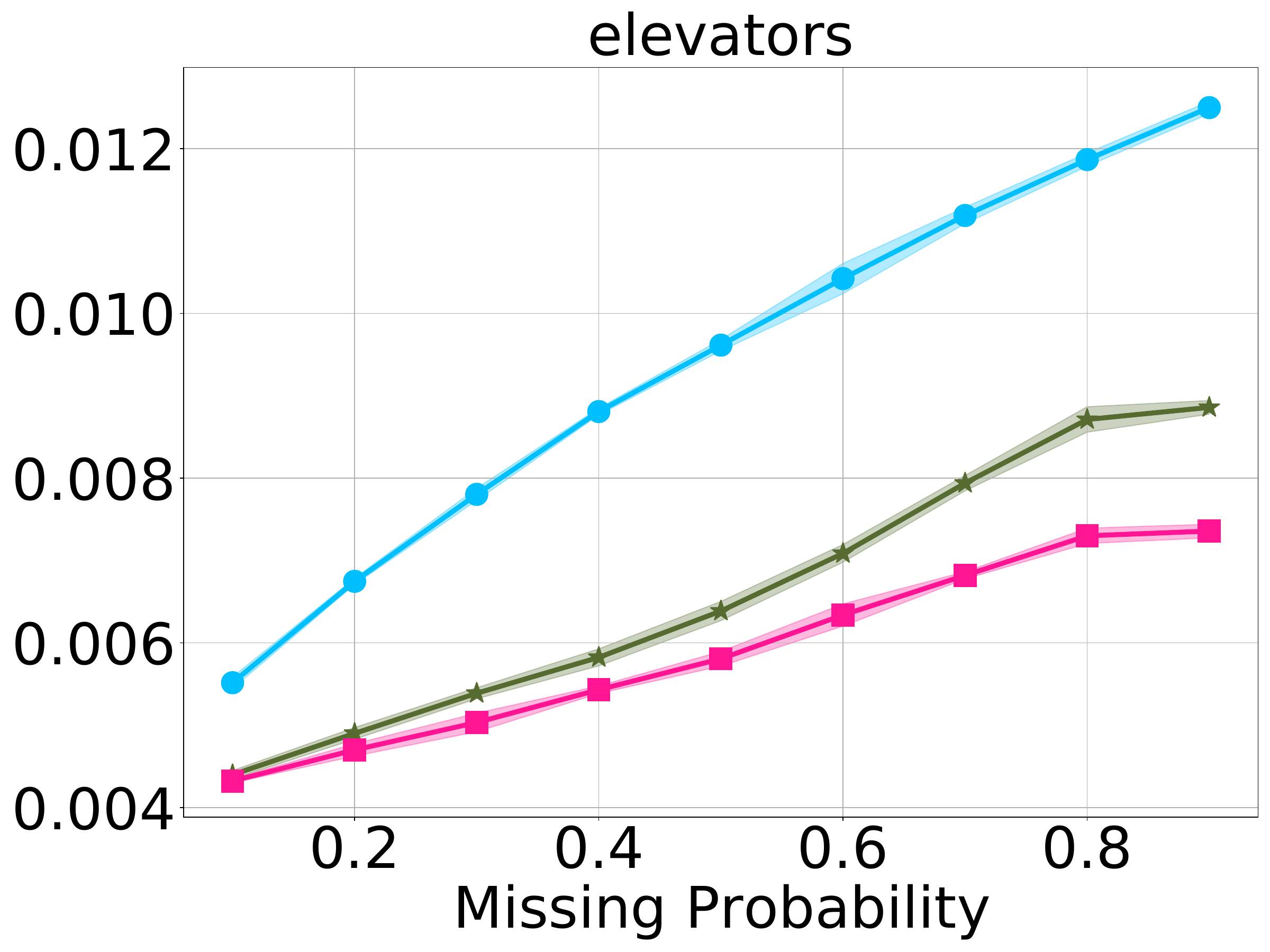}}
        {\includegraphics[width=0.245 \textwidth]{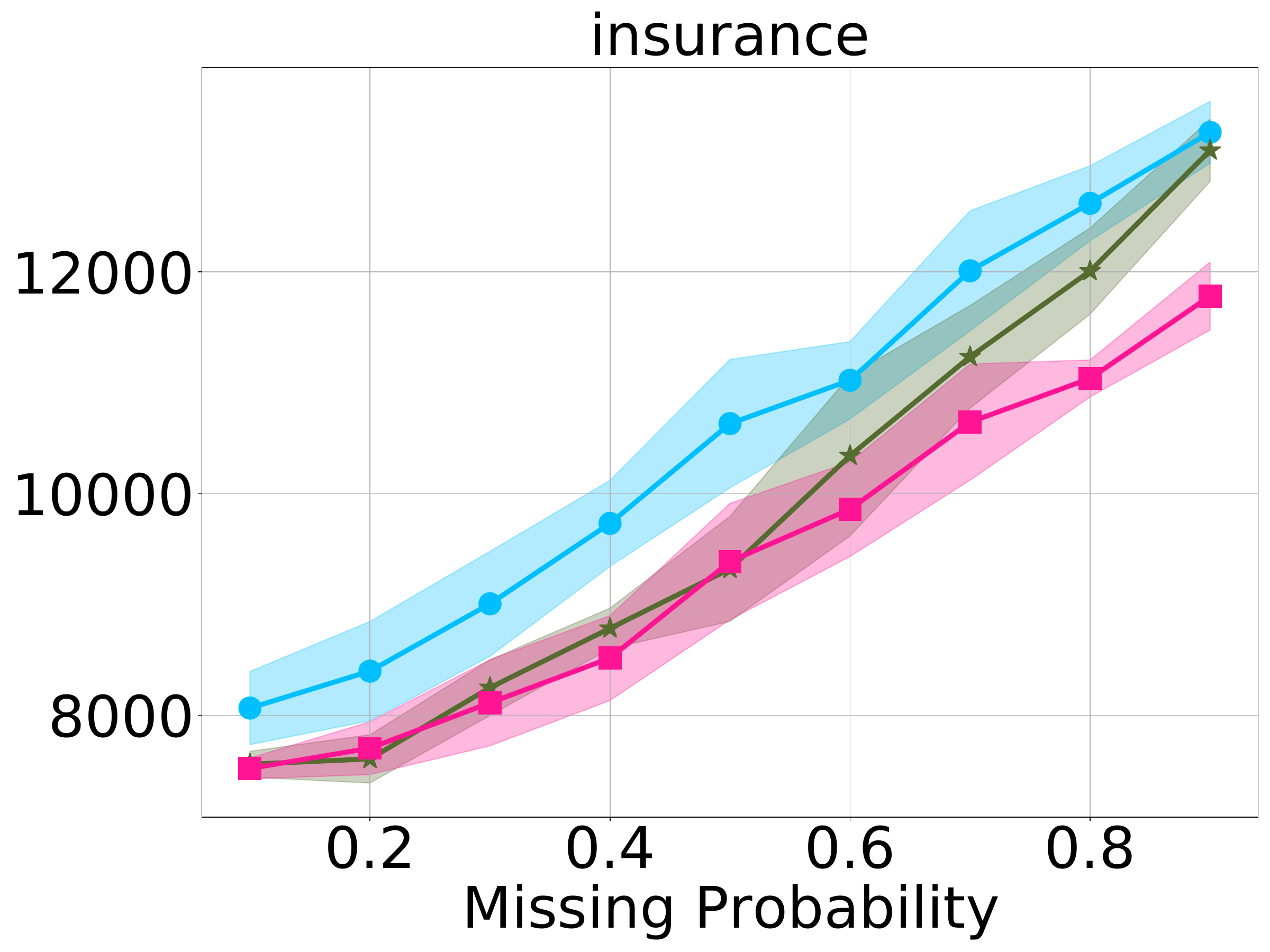}}
    \caption{Evaluating RMSE (y-axis) of the predictions of SVR under different percentages of missing features (x-axis) over four real-world regression datasets. Overall, our expected predictions outperform median imputation and MAP.
    }
    \label{fig: svr-exp}
\end{figure*}

\section{Related Work}

\label{sec: related}
The idea of composing kernels with sums and products first emerged in the literature of the automatic statistician, and is applied to structure discovery for Gaussian processes and nonparametric regression tasks~\citep{duvenaud2013structure}. Compositional kernel machines~\citep{gens2017compositional} further leverage sum-product functions~\citep{friesen2016sum} for a tractable instance-based method for object recognition. 
Instead, we provide the general theoretical foundations for the tractable computation of expected kernels.

Our proposed BBIS scheme extends the original black-box importance sampling to discrete domains, which have not been explored yet, contrary to the continuous case~\citep{cockayne2019bayesian,oates2014control}.
Alternatives to black-box optimization include directly approximating the proposal distribution to compute the importance weights~\citep{delyon2016integral}.
The KSD~\citep{liu2016stein,liu2016kernelized} and its variants~\citep{yang2018goodness,wang2019stein,wang2018stein,singhal2019kernelized}, when applied to particle-based inference, consider the particles to be fully instantiated while our proposed conditional KDSD generalizes it to collapsed~particles.

Closely related, works in probabilistic graphical models represent collapsed particles by circuits.
The approximate compilation proposed by \citet{friedman2018approximate} employs online collapsed importance sampling (CIS) partially compiling the target distribution into a sentential decision diagram~(SDD)~\citep{darwiche2011sdd}. %
\citet{rahman2019cutset} propose to use a cutset network, a smooth, decomposable and deterministic PC to distill a collapsed Gibbs sampling (CGS) scheme for Bayesian networks.  
Arithmetic circuits~\citep{darwiche2003differential}, other kinds of PCs that can be compiled from Bayesian networks have been used in the context of variational approximations~\citep{lowd2010approximate,vlasselaer2015anytime,shih2020probabilistic}.

\section{Empirical Evaluation}
\label{sec: empirical evaluation} 

\begin{figure*}[!t]
    \centering
        {\includegraphics[width=0.23 \textwidth]{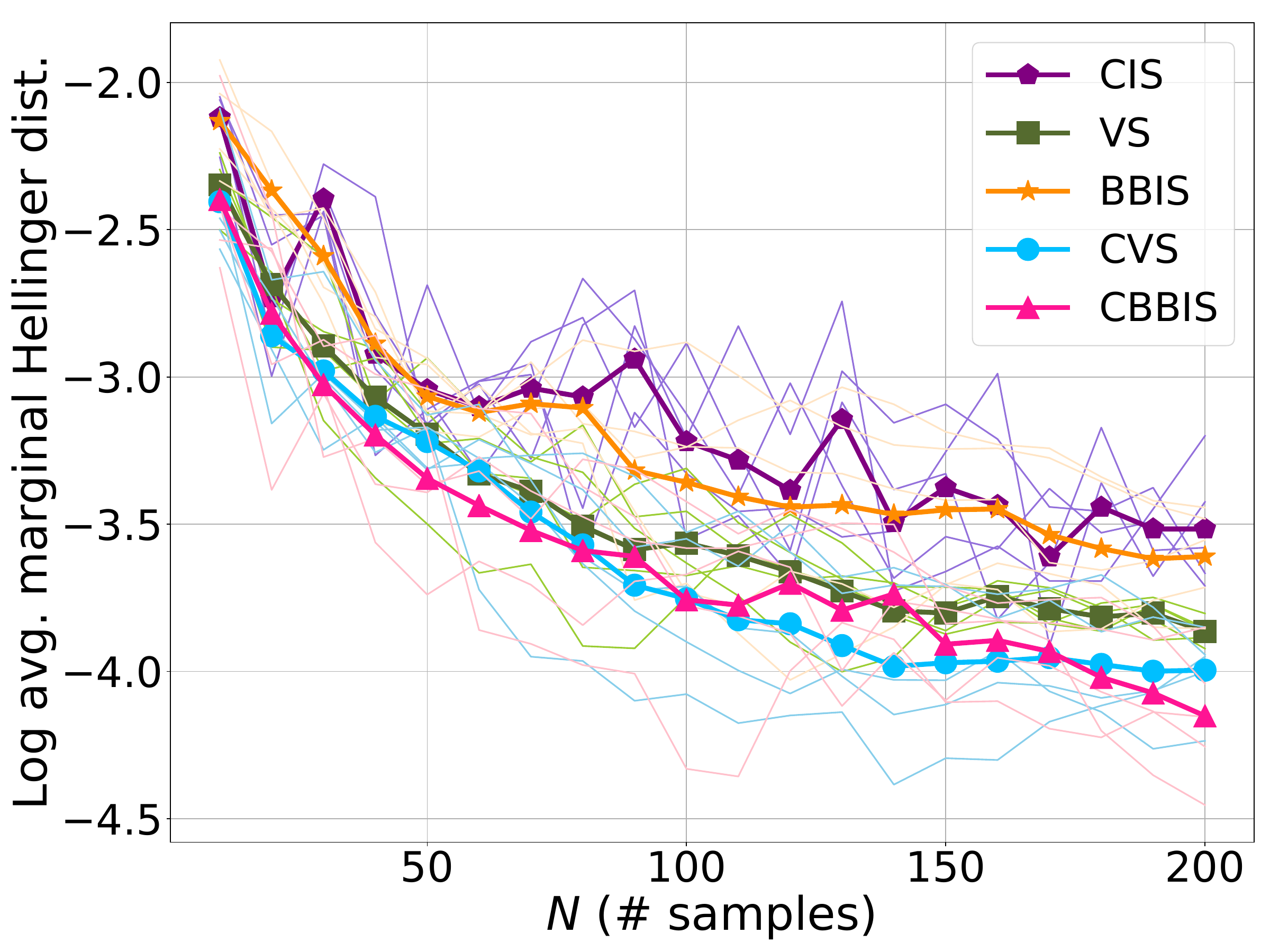}}
        {\includegraphics[width=0.23 \textwidth]{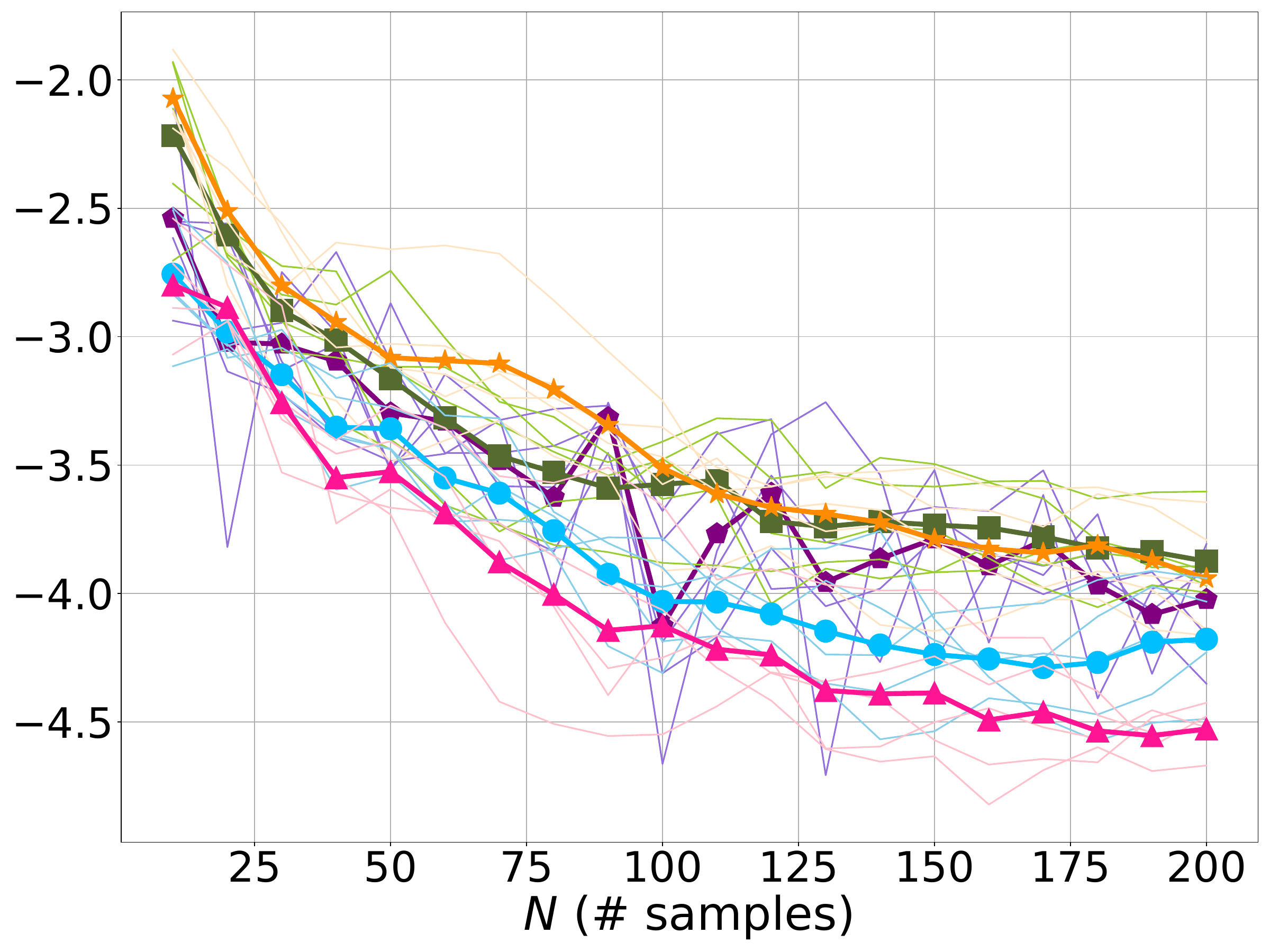}}
        {\includegraphics[width=0.23 \textwidth]{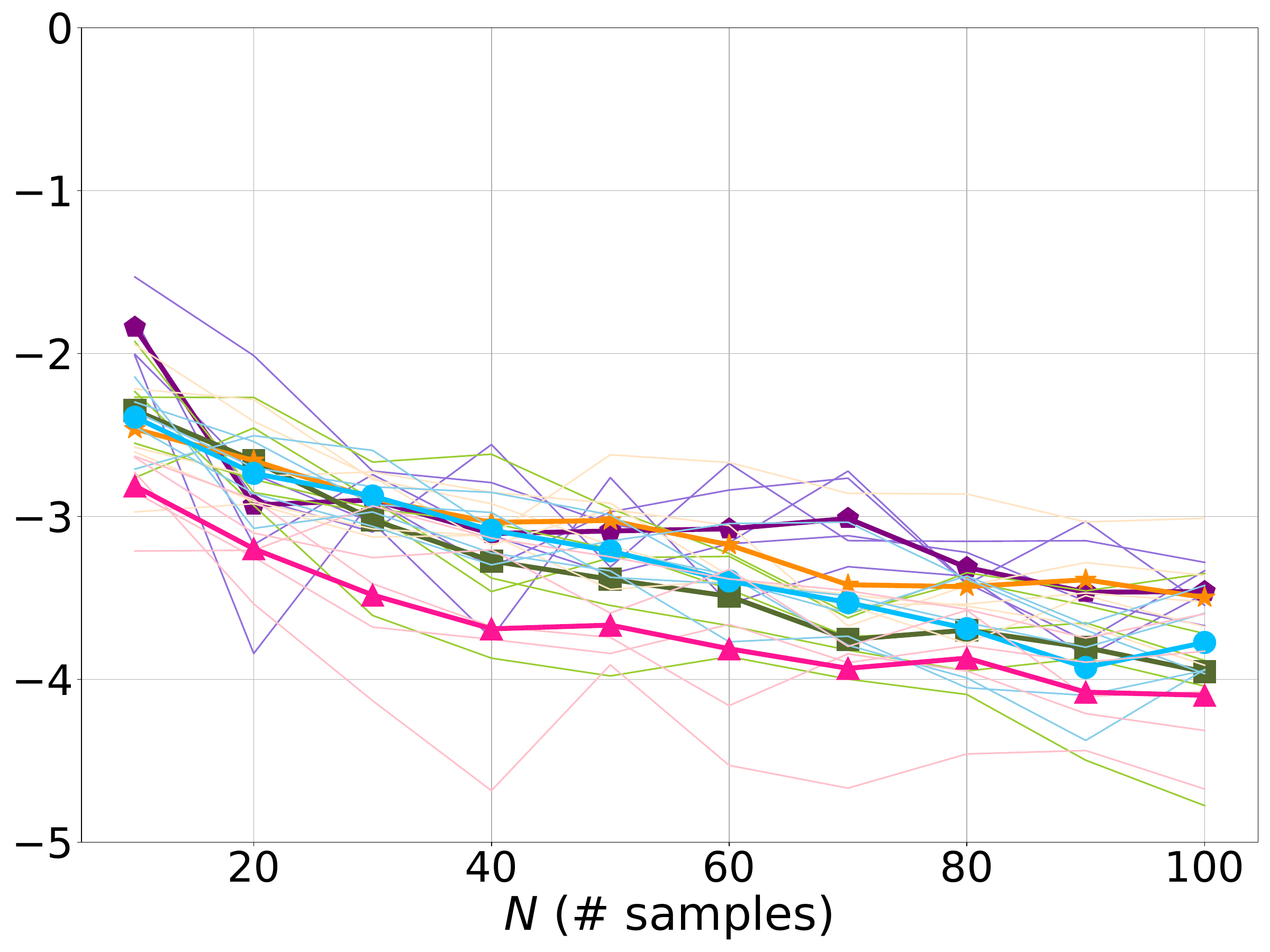}}
        {\includegraphics[width=0.23 \textwidth]{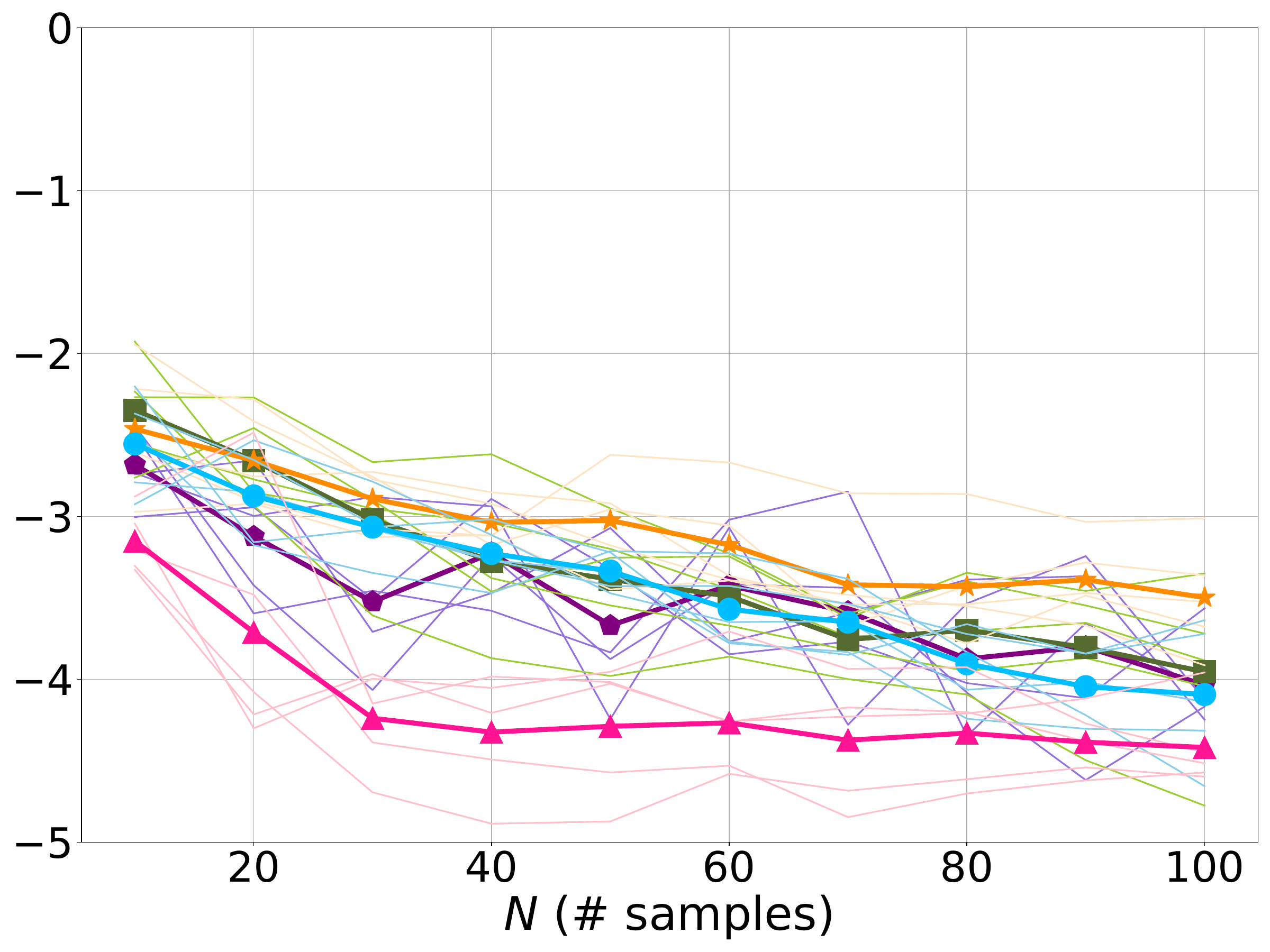}}
    \caption{Log average marginal Hellinger distance (y-axis) vs. different sample sizes ($N$, x-axis),
    evaluated on an Ising or ASIA model as a target distribution ($p$)
    with Gibbs chain as a proposal distribution ($q$). The target distribution and the percentage of collapsed variables are (from left to right): (Ising, $25\%$);  (Ising, $50\%$); (ASIA, $25\%$);  (ASIA, $50\%$).}
    \label{fig: svi-exp}
\end{figure*}

In this section, we empirically evaluate our two novel algorithms, and show how tractable expected kernels can benefit scenarios where the kernels serve as \emph{embedding for features} and \emph{embedding for distributions}.\footnote{Code for reproducing our empirical evaluation can be found at \url{github.com/UCLA-StarAI/ExpectedKernels}} 
We provide preliminary experiments to answer the following questions: 
\textbf{(Q1)} Do expected predictions at deployment time improve predictions over common imputation techniques to deal with missingness for SVR? 
\textbf{(Q2)} 
How is the performance of CBBIS when compared to other IS methods?
\textbf{(Q3)}
How much does collapsing more variables improve estimation quality?

\subsection{Regression Under Missing Data}

We compare our expected prediction with median imputation techniques 
and another natural and strong baseline: imputing missing values by MAP inference over the learned data distribution.
We evaluate all competitors on four common regression benchmarks from several domains following \citet{khosravi2019tractable}. 
For each benchmark, we adopt the \textsc{Strudel} algorithm~\citep{DangPGM20} to learn structured-decomposable and deterministic PCs from data to represent the data distributions. \textsc{Strudel} initializes from a Chow-Liu tree~\citep{chow1968approximating}. Then the structure learning is performed by doing heuristic-based greedy search over possible structures. Intuitively, it iteratively models the data with variable heuristic and edge heuristic. 
Recall from Section~\ref{sec: circuit representation} that deterministic PCs can perform exact MAP inference in polytime and thus the MAP imputation can be done tractably.

For the missingness setting, we assume data to be MCAR with missing probability $\pi \in \{0.1, 0.2, \dots, 0.9\}$, each of which reports the average result over five independent trials. 
We employ RBF kernels, which are naturally compatible with any structured-decomposable PCs (see Section~\ref{sec: tractable computation of expected kernels}).

Figure~\ref{fig: svr-exp} summarizes our results: we can answer \textbf{Q1} in a positive way since expected prediction performs equally well or better than other imputation methods. This is because expected prediction computes the exact expectation over expressive distributions while other imputation techniques consider a single possible completion and make additional restrictive distributional assumptions.

\subsection{Approximate inference via CBBIS}
\label{sec: inference performance}
We empirically evaluate our CBBIS scheme against different baselines on some synthetic benchmarks where we can exactly measure approximation quality.
For each baseline, we measure the quality of the estimated marginals for each variable against a ground truth target distribution represented as an Ising model on a $4\times4$ grid whose potentials have been randomly generated. To show our methods are suitable for different graph structures, we also test on the Bayesian network ASIA \citep{lauritzen1988local}.
We report in log scale the average Hellinger distance between estimated marginals and ground-truth marginals across all variables over five runs. 

We compare our proposed CBBIS in Algorithm~\ref{alg: cbbis-pc} against the following baselines: 
a vanilla Gibbs sampler~(VS),
a collapsed Gibbs sampling scheme~(CVS),
Categorical black-box importance sampling~(BBIS), 
and online collapsed importance sampling~(CIS) proposed by \citet{friedman2018approximate},\footnote{\url{github.com/UCLA-StarAI/Collapsed-Compilation}} cf.~\cref{sec: related}.
For both BBIS and CBBIS we use Gibbs chains as proposal mechanisms.
Note that CIS employs a different and adaptive proposal scheme where new samples and variables to be collapsed are heuristically selected by computing marginals via the SDD that compiles the collapsed distribution. 

For the kernel function in KDSD, we follow the kernel choice in \cite{yang2018goodness}, that is, the exponential Hamming kernel.
The quadratic programming problem to retrieve the optimal weights in BBIS and CBBIS is solved by CVXOPT~\citep{vandenberghe2010cvxopt}.
To obtain the PC representation of collapsed samples,
we use the compilation algorithm by \cite{shen2016tractable}
for collapsed samples in both CBBIS and CVS. 
The compilation step is fast.
For each collapsed sample, the compilation algorithm translates the conditional Ising model and the conditional Bayesian networks in our case to structured decomposable PCs within seconds.
For CIS, we adopt the default compilation algorithm in its implementation.
We collapse $25\%$ and $50\%$ of the variables for methods exploiting collapsed samples: CVS, CIS and CBBIS.
Figure~\ref{fig: svi-exp} summarizes our results: we can answer \textbf{Q2}
in a positive way since CBBIS performs equally well or better than other baselines.
Moreover, for \textbf{Q3}, we can see that methods with collapsed samples, CBBIS and CVS, outperform their non-collapsed counterparts, BBIS and VS respectively, i.e., collapsing helps boosting estimation.
It is more evident when collapsing half of the~variables.

\section{Conclusion}
\label{sec: conclusion}
We introduced kernel circuits, which enable us to derive the sufficient structural constraints for a tractable computation of expected kernels. We further demonstrate how this tractable computation gives rise to two novel kernel-embedding based algorithms.

\scalebox{0.01}{Æquis accipiunt animis donantve Corona}\vspace{-20pt}
\begin{acknowledgements}
This work is supported in part by NSF grants \#CCF-1837129, \#IIS-1956441, \#IIS-1943641, DARPA grant \#N66001-17-2-4032, a Sloan Fellowship, and gifts from Intel and Facebook Research. ZZ is supported by a NEC Student Research Fellowship.
\end{acknowledgements}

\bibliography{li_466}

\clearpage

\maketitle

\section{Proofs}
\label{sec: proofs}

We first present another hardness result about the computation of expected kernels besides Theorem~\ref{thm: hardness for expected kernels}.
\begin{thm}
There exist representations of distributions $p$ and $q$ that are smooth and compatible, yet computing the expected kernel of a simple kernel $k$ that is the Kronecker delta is already \#P-hard.
\end{thm}

\begin{proof}{(an alternative proof to the one in Section~\ref{sec: tractable computation of expected kernels})}
Consider the case when the positive definite kernel $k$ is a Kronecker delta function defined as $k(\x, \xp) = 1$ if and only if $\x = \xp$. 
Moreover, assume that the probabilistic circuit $p$ is smooth and decomposable, and that $q = p$.
Then computing the expected kernel is equivalent to computing the power of a probabilistic circuit $p$, that is,
$\doublesum_k(p, q) = \sum_{\x \in \xdomain} p^2(\x)$
with $\xdomain$ being the domain of variables $\X$.
\citet{vergari2021compositional} proves that the task of computing $\sum_{\x \in \xdomain} p^2(\x)$ is \#P-hard even when the PC $p$ is smooth and decomposable, which concludes our proof.
\end{proof}

\paragraph{Proposition~\ref{pro: recursive-sum-nodes}}
Let $\p_{n}$ and $\q_{m}$ be two compatible probabilistic circuits over variables $\X$ whose output units $n$ and $m$ are sum units,
denoted by
$\p_{n}(\X) = \sum_{i\in\ch(n)}\theta_{i}\p_{i}(\X)$ and 
$\q_{m}(\X) = \sum_{j\in\ch(m)}\delta_{j}\q_{j}(\X)$ respectively. 
Let $k_{l}$ be a kernel circuit with its output unit being a sum unit $l$, denoted by 
$k_l(\X) = \sum_{c \in \ch(l)} \gamma_{c} k_{c}(\X)$.
Then it holds that
\begin{equation}
    \doublesum_{k_l}(\p_n, \q_m) = 
    \sum_{i\in\ch(n)}\theta_{i}\sum_{j\in\ch(m)}\delta_{j} 
    \sum_{c \in \ch(l)} \gamma_{c}
    \,\doublesum_{k_c}(\p_i, \q_j).
\end{equation}

\begin{proof}
$\doublesum_{k_l}(\p_n, \q_m)$ can be expanded as
\begin{align*}
    &\quad\doublesum_{k_l}(p_n, q_m) \\
    &= \sum_\x \sum_{\xp} p_n(\x) q_m(\xp) k_l(\x,\xp) \\
    &= \sum_\x \sum_{\xp} \sum_{i\in\ch(n)}\theta_{i}\p_{i}(\x) \sum_{j\in\ch(m)}\delta_{j}\q_{j}(\xp) \sum_{c\in\ch(l)} \gamma_{c} k_c(\x,\xp) \\
    &= \sum_{i\in\ch(n)}\theta_{i}\sum_{j\in\ch(m)}\delta_{j} 
    \sum_{c \in \ch(l)} \gamma_{c}
    \,\doublesum_{k_c}(\p_i, \q_j).
\end{align*}
\end{proof}

\paragraph{Proposition~\ref{pro: recursive-product-nodes}}
Let $\p_n$ and $\q_m$ be two compatible probabilistic circuits over variables $\X$ whose output units $n$ and $m$ are product units, denoted by $\p_{n}(\X) = \p_{n_\leftn}(\X_{\leftn})\p_{n_\rightn}(\X_{\rightn})$ and $\q_{m}(\X) = \q_{m_\leftn}(\X_{\leftn})\q_{m_\rightn}(\X_{\rightn})$.
Let $k$ be a kernel circuit that is kernel-compatible with the circuit pair $\p_n$ and $\q_m$ with its output unit being a product unit denoted by
$k(\X, \X^\prime) = k_{\leftn}(\X_{\leftn}, \X_{\leftn}^\prime) k_{\rightn}(\X_{\rightn}, \X_{\rightn}^\prime)$.
Then it holds that
\begin{equation*}
    \doublesum_{k}(\p_n, \q_m) = \doublesum_{k_\leftn}(\p_{n_\leftn}, \q_{m_\leftn}) \cdot \doublesum_{k_\rightn}(\p_{n_\rightn}, \q_{m_\rightn}).
\end{equation*}

\begin{proof}
$\doublesum_{k}(p_n, q_m)$ can be expanded as
\begin{align*}
    &\quad\doublesum_{k}(p_n, q_m) \\
    &= \sum_\x \sum_{\xp} p_n(\x) q_m(\xp) k(\x,\xp) \\
    &= \sum_\x \sum_{\xp} \p_{m_\leftn}(\x_{\leftn})\p_{m_\rightn}(\x_{\rightn}) \q_{n_\leftn}(\x_{\leftn})\q_{n_\rightn}(\x_{\rightn}) k_{\leftn}(\x_\leftn,\xp_\leftn) k_{\rightn}(\x_\rightn,\xp_\rightn) \\
    &= \doublesum_{k_\leftn}(\p_{n_\leftn}, \q_{m_\leftn}) \cdot \doublesum_{k_\rightn}(\p_{n_\rightn}, \q_{m_\rightn}).
\end{align*}
\end{proof}

\paragraph{Corollary~\ref{cor: tractable mmd}.} 
Following the assumptions in Theorem~\ref{thm: double sum complexity},
the squared maximum mean discrepancy $\mmd[\mathcal{H}, p, q]$ in RKHS $\mathcal{H}$ associated with kernel $k$ as defined in \citet{gretton2012kernel} can be tractably computed.

\begin{proof}
This is an immediate result following Theorem~\ref{thm: double sum complexity} by rewriting MMD as defined in \citet{gretton2012kernel} in the form of a linear combination of expected kernels, that is,
$\mmd^2[\mathcal{H}, p, q] = \doublesum_k(p, p) + \doublesum_k(q, q) - 2 \doublesum_k(p, q)$.
\end{proof}

\paragraph{Corollary~\ref{cor: tractable kdsd}.}
Following the assumptions in Theorem~\ref{thm: double sum complexity}, if the probabilistic circuit $p$ further satisfies determinism,
the kernelized discrete Stein discrepancy~(KDSD)
$\mathbb{D}^2(q~\parallel~p) = \expec_{\x, \xp \sim q}[k_p(\x, \xp)]$
in the RKHS associated with kernel $k$ as defined in \citet{yang2018goodness} can be tractably computed.

Before showing the proof for Corollary~\ref{cor: tractable kdsd}, we first give definitions that are necessary for defining KDSD as follows to be self-contained.
\begin{mydef}[Cyclic permutation]
For a finite set $\xdomain$ and $D = |\xdomain|$, a cyclic permutation $\neg: \xdomain \rightarrow \xdomain$ is a bijective function such that for some ordering $a_1, a_2, \cdots, a_{D}$ of the elements in $\xdomain$, $\neg a_i = a_{(i + 1) \mod D}$, $\forall i = 1, 2, \cdots, D$.
\end{mydef}

\begin{mydef}[Partial difference operator]
\label{def: partial diff op}
For any function $f: \xdomain \rightarrow \mathbb{R}$ with $D = |\xdomain|$,
the partial difference operator is defined as
\begin{equation}
\Delta^*_i f(\X) := f(\X) - f(\neg_i \X), \forall i = 1, \cdots, D,
\end{equation}
with $\neg_i \X := (X_1, \cdots, \neg X_i, \cdots, X_D)$.
Moreover, the difference operator is defined as $\Delta^* f(\X) := (\Delta^*_1 f(\X), \cdots, \Delta^*_D f(\X))$.
Similarly, let $\invneg$ be the inverse permutation of $\neg$, and $\Delta$ denote the difference operator defined with respect to $\invneg$, i.e.,
\begin{equation*}
    \Delta_i f(\X) := f(\X) - f(\invneg_i \X), i = 1, \cdots, D.
\end{equation*}
\end{mydef}

\begin{mydef}[Difference score function]
The (difference) score function is defined as $\score(\X) := \frac{\Delta^* p(\X)}{p(\X)}$ on domain $\xdomain$ with $D = \mid \xdomain \mid$, a vector-valued function with its $i$-th dimension being
\begin{equation}
\scorei(\X) := \frac{\Delta^*_i p(\X)}{p(\X)} = 1 - \frac{p(\neg_i \X)}{p(\X)},
i = 1, 2, \cdots, D.
\end{equation}
\end{mydef}

Given the above definitions, 
the discrete \textit{Stein discrepancy} between two distributions $p$ and $q$ is defined as
\begin{equation}
    \mathbb{D}(q \parallel p) := \sup_{\vv{f} \in \mathcal{H}} \expec_{\x\sim q(\X)}[\steinop \vv{f}(\x)],
    \label{eq:sd}
\end{equation}
where $\vv{f}: \xdomain \rightarrow \mathbb{R}^D$ is a \textit{test} function, belonging to 
some function space $\mathcal{H}$ and $\steinop$ is the so-called \textit{Stein difference operator}, which is defined as 
\begin{equation}
    \steinop \vv{f} = \score(\x) \vv{f}^\top - \Delta \vv{f}(\x).
    \label{eq:disc-stein-op}
\end{equation}

If the function space $\mathcal{H}$ is an reproducing kernel Hilbert space (RKHS) on $\xdomain$ equipped with a kernel function $k(\cdot, \cdot)$, then a \textit{kernelized discrete Stein discrepancy}~(KDSD) is defined and admits a closed-form representation as
\begin{equation}
    \ksd(q \parallel p) 
    := \mathbb{D}^2(q \parallel p)
    = \expec_{\x, \xp \sim q}[k_p(\x, \xp)].
    \label{eq:ksd}
\end{equation}
Here, the kernel function $k_p$ is defined as
\begin{equation*}
\label{eq:kernel-p}
\begin{split}
    k_p(\x, \xp) 
    &= \score(\x)^\top k(\x, \xp) \score(\xp) - \score(\x)^\top \Delta^{\xp} k(\x, \xp) \\
    &\quad - \Delta^{\x} k(\x, \xp)^\top \score(\xp) + \trace(\Delta^{\x,\xp} k(\x, \xp)),
\end{split}
\end{equation*}
where the difference operator $\Delta^{\x}$ is as in Definition~\ref{def: partial diff op}. The superscript $\x$ specifies the variables that it operates on.

\begin{proof}{[Corollary~\ref{cor: tractable kdsd}]}
By the definition of difference score functions, the close form of KDSD can be further rewritten as follows.
\begin{equation}
\label{eq: rewrite kdsd}
\begin{split}
&\expec_{\x, \xp \sim q}[k_p(\x, \xp)] \\
= & \sum_{i = 1}^D \expec_{\x, \xp \sim q}[
\frac{p(\neg_i \x)p(\neg_i \xp)}{p(\x)p(\xp)} k(\x, \xp)
- \frac{p(\neg_i \x)}{p(\x)} k(\x, \neg_i \xp) \\
&\quad - \frac{p(\neg_i \xp)}{p(\xp)} k(\neg_i \x, \xp)
+ k(\neg_i \x, \neg_i \xp)
]\\
=& \sum_{i = 1}^D [ \doublesum_k(q\frac{\tilde{p}_i}{p}, q\frac{\tilde{p}_i}{p}) 
- \doublesum_k(q\frac{\tilde{p}_i}{p}, \tilde{q}_i) \\
&\quad - \doublesum_k(\tilde{q}_i, q\frac{\tilde{p}_i}{p})
+ \doublesum_k(\tilde{q}_i, \tilde{q}_i)]
\end{split}
\end{equation}
where $D$ denotes the cardinality of the domain of variables $\X$, the probablity $\tilde{p}_i(\X) := p(\neg_i \X)$ and the probablity $\tilde{q}_i(\X) := q(\neg_i \X)$.
Notice that the cyclic permutation $\neg_i$ operates on individual variable and the resulting PC $\tilde{p}_i$ and $\tilde{q}_i$ retains the same structure properties as PCs $p$ and $q$ respectively.
To prove that KDSD can be tratably computed, it suffices to prove that the expected kernel terms in Equation~\ref{eq: rewrite kdsd} can be tractably computed.

For a deterministic and structured-decomposable PC $p$, since PC $\tilde{p}_i$ retains the same structure, then resulting ratio $\tilde{p}_i / p$ is again a smooth circuit compatible with $p$ by \citet{vergari2021compositional}.
Moreover, since PC $p$ and $q$ are compatible, the circuit $\tilde{p}_i / p$ is compatible with PC $q$.
Thus, the resulting product $q\frac{\tilde{p}_i}{p}$ is a circuit that is smooth and compatible with both $p$ and $q$ by Theorem B.2 and thus compatible with $\tilde{q}_i$.
By similar arguments, we can verify that all the circuit pair in the expected kernel terms in Equation~\ref{eq: rewrite kdsd} satisfy the assumptions in Theorem~\ref{thm: double sum complexity} and thus they are amenable to the tractable computation we propose in Algorithm~\ref{alg: double-sum}, which finishes our proof.

\end{proof}

\paragraph{Proposition (convergence of Categorical BBIS).} Let $f(\x)$ be a test function. Assume that $f - \expec_p[f] \in \mathcal{H}_p$, with $\mathcal{H}_p$ being the RKHS associated with the kernel function $k_p$, and $\sum_{i} w_i = 1$, then it holds that
\begin{align*}
    \left| \sum_{n = 1}^N w_n f(x_n) - \expec_p f \right| \leq C_f \sqrt{\ksd(\{\x^{(n)}, w_n\} \parallel p)},
\end{align*}
where $C_f := \parallel f - \expec_p f \parallel_{\mathcal{H}_p}$.
Moreover, the convergence rate is $\bigO(N^{-1/2})$.

\begin{proof}
    Let $\hat{f}(\x) := f(\x) - \expec_p f $, then it holds that
    \begin{equation*}
        \begin{split}
            \left| \sum_{n = 1}^N w_n f(\x^\idx{n}) - \expec_p f \right|
            &= \left|  \sum_{n = 1}^N w_n \hat{f}(\x^\idx{n}) \right| \\
            &= \left|  \sum_{n = 1}^N w_n \langle \hat{f}, k_p(\cdot, \x^\idx{n}) \rangle \right| \\
            &= \left|  \langle \hat{f}, \sum_{n = 1}^N w_n  k_p(\cdot, \x^\idx{n}) \rangle_{\mathcal{H}_p} \right| \\
            &\leq \parallel \hat{f} \parallel_{\mathcal{H}_p} \cdot \parallel \sum_{n = 1}^N w_n  k_p(\cdot, \x^\idx{n}) \parallel_{\mathcal{H}_p} \\
            &= \parallel \hat{f} \parallel_{\mathcal{H}_p} \cdot \sqrt{\ksd(\{\x^\idx{n}, w_n \} \parallel p)}.
        \end{split}
    \end{equation*}
    We further prove the convergence rate of the estimation error by using the importance weights as reference weights.
    Let $v_n^* = \frac{1}{n} p(\x^\idx{n}) / q(\x^\idx{n})$. 
    Then $\ksd(\{\x^\idx{n}, v_n^*\} \parallel p)$ is a degenerate V-statistics~\citep{liu2016black} and it holds that $\ksd(\{\x^\idx{n}, v_n^*\} \parallel p) = \bigO(N^{-1})$.
    Moreover, we have that $\sum_{n = 1}^N v^*_n = 1 + \bigO(N^{- 1/2})$, which we denote by $Z$, i.e., $Z = \sum_{n = 1}^N v^*_n$.
    Let $w^*_n = v^*_n / Z$,
    then it holds that
    \begin{equation*}
        \begin{split}
            \ksd(\{\x^\idx{n}, w_n^*\} \parallel p) = \frac{\ksd(\{\x^\idx{n}, v_n^*\} \parallel p)}{Z^2} = \bigO(N^{-1}).
        \end{split}
    \end{equation*}
    Therefore, 
    \begin{equation*}
        \begin{split}
            \left| \sum_{n = 1}^N w_n f(\x^\idx{n}) - \expec_p f \right|
            &\leq \parallel \hat{f} \parallel_{\mathcal{H}_p} \cdot \sqrt{\ksd(\{\x^\idx{n}, w_n \} \parallel p)} \\
            &\leq \parallel \hat{f} \parallel_{\mathcal{H}_p} \cdot \sqrt{\ksd(\{\x^\idx{n}, w_n^* \} \parallel p)} \\
            &= \bigO(N^{- 1/2}).
        \end{split}
    \end{equation*}
\end{proof}

\paragraph{Proposition~\ref{pro: tractable conditonal kernel function}.} Let $p(\Xc \mid \xs)$ be a PC that encodes a conditional distribution over variables $\Xc$ conditioned on $\Xs = \xs$, and $k$ be a KC.
If the PC $p(\Xc \mid \xs)$ and $p(\Xc \mid \xs^\prime)$ are compatible and $k$ is kernel-compatible with the PC pair for any $\xs$, $\xs^\prime$, 
then the conditional kernel function $k_{p, \s}$ as defined in Proposition~\ref{pro: kdsd} can be tractably computed.

\begin{proof}
From Proposition~\ref{pro: kdsd}, $k_{p, \s}$ can be written as 
\begin{equation*}
    k_{p, \s} = \sum_{i=1}^D \expec_{\x_\co \sim p(\X_\co \mid \x_\s), \xp_\co \sim p(\X_\co \mid \xp_\s)} [k_{p,i}(\x,\xp)],
\end{equation*}
where $k_{p,i}$ can be expanded as follows.
\begin{equation*}
    \begin{split}
        k_{p,i}(\x,\xp) 
        = &\frac{p(\neg_i\x)p(\neg_i\xp)}{p(\x)p(\xp)}k(\x,\xp)
        - \frac{p(\neg_i\x)}{p(\x)}k(\x,\neg_i\xp) \\
        &- \frac{p(\neg_i\xp)}{p(\xp)}k(\neg_i\x,\xp)
        + k(\neg_i\x,\neg_i\xp).
    \end{split}
\end{equation*}
for any $i \in \co$, given that none of the variables in $\X_\s$ is flipped in the above formulation, kernel $k_{p,i}$ can be further written as
\begin{align*}
    k_{p,i}(\x,\xp) = &\frac{p(\neg_i\x_\co \mid \x_\s)p(\neg_i\xp_\co \mid \xp_\s)}{p(\x_\co \mid \x_\s)p(\xp_\co \mid \xp_\s)}k(\x,\xp) \\
        &- \frac{p(\neg_i\x_\co \mid \x_\s)}{p(\x_\co \mid \x_\s)}k(\x,\neg_i\xp) \\
        &- \frac{p(\neg_i\xp_\co \mid \xp_\s)}{p(\xp_\co \mid \xp_\s)}k(\neg_i\x,\xp) \\
        &+ k(\neg_i\x,\neg_i\xp).
\end{align*}
By substituting $k_{p,i}$ into the expected kernel in the expectation of $k_{p,i}$ with respect to the conditional distributions can be simplified to be a constant zero, that is,
\begin{align*}
    &\expec_{\x_{\co} \sim p(\X_\co \mid \x_\s), \xp_{\co} \sim p(\X^\prime_\co \mid \xp_\s) }[k_{p, i}(\x, \xp)] = ~0.
\end{align*}
Thus, $k_{p, \s}$ can be expanded as
\begin{align*}
    k_{p,\s}(\x,\xp) &=  \expec_{\x_\co \sim p(\X_\co \mid \x_\s), \xp_\co \sim p(\X_\co \mid \xp_\s)} [\sum_{i \in \s} k_{p,i}(\x,\xp)] \\
        &= \sum_{i\in\s} [\frac{p(\neg_i\x_\s)p(\neg_i\xp_\s)}{p(\x_\s)p(\xp_\s)} \cdot \doublesum_{k(\cdot,\cdot)} (p(\cdot\mid\neg_i\x_\s), p(\cdot\mid\neg_i\xp_\s)) \\
        &\quad-  \frac{p(\neg_i\x_\s)}{p(\x_\s)}\cdot\doublesum_{k(\cdot,\neg_i\cdot)} (p(\cdot\mid\neg_i\x_\s), p(\cdot\mid\xp_\s)) \\
        &\quad - \frac{p(\neg_i\xp_\s)}{p(\xp_\s)}\cdot\doublesum_{k(\neg_i\cdot,\cdot)}(p(\cdot\mid\x_\s),p(\cdot\mid\neg_i\xp_\s)) \\
        &\quad+ \doublesum_{k(\neg_i\cdot,\neg_i\cdot)}(p(\cdot\mid\x_\s),p(\cdot\mid\xp_\s))].
\end{align*}
As Theorem~\ref{thm: double sum complexity} has shown that $\doublesum_k(\p,\q)$ can be computed exactly in time linear in the size of each PC, $k_{p,\s}(\x,\xp)$ can also be computed exactly in time $\bigO(|\p_1||\p_2||k|)$, where $\p_1$ and $\p_2$ denote circuits that represent the conditional probability distribution given the index set, i.e., $p(\cdot \mid \x_\s)$ or $p(\cdot \mid \neg_i\x_\s)$.
\end{proof}

\section{Algorithms}

Algorithm~\ref{alg: bbis} summarizes how to perform the BBIS scheme we propose for Categorical distributions, and generate a set of weighted samples.
\begin{algorithm}[t]
\caption{\textsc{CategoricalBBIS}($p, q, k, n$)} 
\label{alg: bbis} 
\textbf{Input:} target distributions $p$ over variables $\X$, a black-box mechanism $q$, a kernel function $k$ and number of samples $n$ \\
\textbf{Output:} weighted samples $\{(\x^\idx{i}, w^*_i)\}_{i=1}^{n}$
\begin{algorithmic}[1]
\State Sample $\{\x^\idx{i}\}_{i=1}^n$ from $q$
\For{ $i = 1,\ldots,n$}
\For{ $j = 1,\ldots,n$}
\State $[\vv{K}_{p}]_{ij} = k_p(\x^\idx{i}, \x^\idx{j})$ 
\Comment cf.~\cref{eq: kernel-p}
\EndFor
\EndFor
\State $\vv{w}^* = \argmin_{\vv{w}} \left\{ \vv{w}^\top \vv{K_p} \vv{w} \,\middle\vert\, \sum_{i = 1}^n w_i = 1, ~w_i \ge 0 \right\}$
\State \Return $\{(\x^\idx{i}, w^*_i)\}_{i=1}^{n}$
\end{algorithmic} 
\end{algorithm}

\end{document}